\documentclass[smallextended]{svjour3}
\usepackage[utf8]{inputenc}
\usepackage[T1]{fontenc}
\smartqed
\usepackage{latexsym}

\usepackage{cite}
\usepackage{amsmath}
\usepackage{amsfonts}
\usepackage{enumitem}
\usepackage[frozencache,cachedir=.]{minted}
\usepackage{makecell}
\usepackage{bm}
\usepackage{graphicx}
\usepackage[toc,page]{appendix}
\usepackage{subcaption}
\usepackage{url}

\newminted[pycode]{python}{tabsize=4, fontsize=\scriptsize, frame=lines, framesep=\fboxsep, rulecolor=\color{gray!40}}

\DeclareMathOperator*{\argmin}{argmin}
\DeclareMathOperator{\prox}{prox}
\DeclareMathOperator{\dom}{dom}
\DeclareMathOperator{\sgn}{sgn}
\newcommand{\vv}[1]{\mathbf{#1}}
\newcommand{\vx}{\vv{x}}
\newcommand{\vy}{\vv{y}}
\newcommand{\vz}{\vv{z}}
\newcommand{\vu}{\vv{u}}
\newcommand{\va}{\vv{a}}
\newcommand{\vb}{\vv{b}}
\newcommand{\vs}{\vv{s}}
\newcommand{\vA}{\vv{A}}
\newcommand{\defeq}{\mathrel{\mathop:}=}
\newcommand{\half}{\frac{1}{2}}

\newcommand{\extreals}{(-\infty, \infty]}
\newcommand{\reals}{\mathbb{R}}
\mathchardef\mhyphen="2D
\DeclareMathAlphabet{\mymathbb}{U}{BOONDOX-ds}{m}{n}

\begin{document}

\title{Efficient algorithms for implementing incremental proximal-point methods}
\author{Alex Shtoff}
\institute{
  A. Shtoff \at 
  Yahoo Research \\
  Haifa, Israel \\
  \email{alex.shtoff@yahooinc.com}
}
\date{Received: date / Accepted: date}

\maketitle

\begin{abstract}%
Model training algorithms which observe a small portion of the training set in each computational step are ubiquitous in practical machine learning, and include both stochastic and online optimization methods. In the vast majority of cases, such algorithms typically observe the training samples via the gradients of the cost functions the samples incur. Thus, these methods exploit are the \emph{slope} of the cost functions via their first-order approximations. 

To address limitations of gradient-based methods, such as sensitivity to step-size choice in the stochastic setting, or inability to exploit small function variability in the online setting, several streams of research attempt to exploit more information about the cost functions than just their gradients via the well-known proximal operators. However, implementing such methods in practice poses a challenge, since each iteration step boils down to computing the proximal operator, which may not be as easy as computing a gradient. In this work we devise a novel algorithmic framework, which exploits convex duality theory to achieve both \emph{algorithmic efficiency} and \emph{software modularity} of proximal operator implementations, in order to make experimentation with incremental proximal optimization algorithms accessible to a larger audience of researchers and practitioners, by reducing the gap between their theoretical description in research papers and their use in practice. We provide a reference Python implementation for the framework developed in this paper as an open source library at on GitHub\footnote{\url{https://github.com/alexshtf/inc_prox_pt/releases/tag/prox_pt_paper}} \cite{alex_shtof_2024_11500107}, along with examples which demonstrate our implementation on a variety of problems, and reproduce the numerical experiments in this paper. The pure Python reference implementation is not necessarily the most efficient, but is a basis for creating efficient implementations by combining Python with a native backend.
\end{abstract}


\section{Introduction}
Incremental optimization is the `bread and butter' of the theory and practice of modern machine learning, where we aim to minimize a \emph{cost} function of the model's parameters, and can be roughly classified into two major regimes. In the stochastic optimization regime it is assumed that the cost functions are sampled from a stationary distribution, our objective is to design algorithms which minimize the expected cost, the theoretical analysis typically produces bounds on the expected cost, and typical algorithms are variants of the stochastic gradient method \cite{sgd}. In the online regime cost functions from a pre-defined family are chosen by an adversary, our objective is minimizing the cumulative cost incurred by the sequence of the cost functions we observe, the theoretical analysis tool is the performance relatively to a theoretical optimum called \emph{regret}, and typical algorithms are variants of the online gradient method \cite{ogd}, or follow the leader \cite{ftl} based methods operating on linearly approximated costs.

Both in the stochastic and online regime, these methods follow the following incremental protocol: (a) observe the cost function incurred by a small subset of training samples; (b) update the model's parameters. Since the focus of this paper is on the \emph{implementation} of the computational steps of each iteration, the exact regime plays no role, and thus we treat both regimes under the same umbrella, and refer to these methods as \emph{incremental} methods, e.g. incremental gradient descent.

To be concrete, the incremental gradient method having observed the cost function $f$, which is usually interpreted as "make a small step in the descent direction", reads:
\[
\vx_{t+1} = \vx_t - \eta \nabla f(\vx_t),
\]
but can equivalently written using the celebrated \emph{proximal view}:
\[
\vx_{t+1} = \argmin_{\vx} \Bigl \{ f(\vx_t) + \langle \nabla f(\vx_t), \vx - \vx_t \rangle + \frac{1}{2\eta} \|\vx - \vx_t\|_2^2 \Bigr \},
\]
meaning "minimize a linear approximation of $f$, but stay close to $\vx_t$".  A different approach is using the cost directly instead of its linear approximation via the well-known \emph{proximal operator} \cite{moreau_prox_1, moreau_prox_2}:
\begin{equation*}\label{eq:ipp}
\vx_{t+1} = \prox_{\eta f}(\vx_t) \defeq \argmin_{\vx} \Bigl \{ f(\vx) + \frac{1}{2\eta} \|\vx - \vx_t\|_2^2 \Bigr \}, \tag{PROX}
\end{equation*}
The idea is not new, and dates back to the first proximal iteration algorithm of \cite{proxpt_martinet}, which was designed as a theoretical method for the non-incremental regime. In the stochastic optimization regime this idea is known as stochastic proximal point, and in the online regime as implicit online learning. Of course, we do not have to use the cost itself, and $f$ in \eqref{eq:ipp} may be any approximation which is not necessarily linear. In line with our aim to use the same terminology for stochastic and online algorithms, we call the idea above \emph{incremental proximal iteration}. 

In the online regime, incremental proximal iteration was first proposed by \cite{Kivinen1997implicit}, and later analyzed by \cite{kulis2010implicit} where it was termed 'implicit online learning'. Follow-up, such as \cite{karampatziakis2011online, implicit_temporal_variability}, showed improved regret guarantees and robustness in various settings.

In the stochastic regime, incremental proximal iteration was first proposed and analyzed for the finite-sum optimization problem by \cite{bertsekas2011incremental}. However, significant advantages over first order methods were discovered later in the form of improved stability with respect to step-size choices \cite{ryu2014stochastic, aprox, aprox_minibatch} . Consequently, hyper-parameter tuning becomes significantly cheaper when using certain proximal point based methods to train models, and consequently overall the computational cost and energetic footprint of the training process is reduced, even if training for one specific hyper-parameter configuration is more expensive.

Using a finer model instead of a linear approximation bears a cost - since $f$ can be arbitrarily complex, computing the proximal operator may be arbitrarily hard, or even infeasible to do in practice. Thus, there is a trade-off between any advantage a finer approximation may provide, and the difficulty of proximal operator computation.

In contrast to the above-mentioned works, rather than theoretical analysis of novel high-level techniques, our aim is to substantially reduce the difficulty of computing the proximal operator devising efficient algorithms to do so in a variety of setups, and providing a Python reference implementing those algorithms. Our pure Python implementation is not necessarily the most efficient, since its main focus is demonstration and readability, and in terms of speed we aim to be modest: up to a moderate constant factor (10-20 times) slower than a competing gradient method. For some families of functions $f$, computing $\prox_{\eta f}$ may even be as cheap as a regular gradient step, but others require solving a one-dimensional or a low-dimensional optimization problem, and a naive Python implementation is far from optimal. Our Python implementation is designed to demonstrate that efficient algorithms for computing the proximal step need not be hard to implement, and encourage our readers to adapt or revise our implementation to their needs. It is our aim, and hope, that the above contributions make it easier to apply these methods in practice for researchers, while also motivate more research by making numerical experimentation with incremental proximal point methods easily accessible to the research community. 

We'd like to emphasize that the aim of this paper is devising efficient algorithms for \emph{implementing} proximal operators of useful functions in machine learning, rather than \emph{using} proximal operators to derive better converging algorithms, in order to make such implementations accessible to the research community. In addition to efficiency, we use extensive mathematical theory to derive \emph{modular} implementations, so that researchers may compose the functions they desire from atomic building blocks. For example, an L1 regularized logistic regression problem's cost function is composed of a linear function, composed onto the convex logistic function $t \to \ln(1+\exp(t))$, and with the L1 regularizer. Thus, the code for training such models may look like this:
\begin{pycode}
x = torch.zeros(x_star.shape)
# note the modular composition below
optimizer = IncRegularizedConvexOnLinear(x, Logistic(), L1Reg(0.01))  
epoch_loss = 0.
for t, (a, b) in enumerate(my_data_set, start=1):
    step_size = 1. / math.sqrt(t)
    epoch_loss += optimizer.step(step_size, a, b)
\end{pycode}

Clearly, a proximal step is more computationally demanding than a gradient step, thus our implementation is slower. We show, empirically, that for a variety of problems, our implementation is a small constant factor slower than (proximal) gradient steps, which is useful for both researchers and practitioners.  From a practical perspective, since the works we cited above demonstrate, both  theoretically and empirically, that proximal point methods methods are very robust to the step-size choice, hyper-parameter tuning becomes significantly cheaper, and thus such a library of implementations may be useful in practice for \emph{reducing} the computational resources required for training a model. For researchers, an implementation whose run-time is a small constant factor that of gradient descent is useful for conducting numerical experiments for a reasonable computational cost for their new algorithms based on proximal operators. The results in this paper were obtained on a 2019 MacBook Pro with a 2.4 GHz 8-Core Intel Core i9 processor, and 32GB of RAM. The code for reproducing the results is in our code repository.

Since this paper is also aimed at an audience partially unfamiliar with convex analysis theory, but coming from a more machine-learning oriented background. Hence, we will briefly introduce the concepts we need throughout the paper, and refer to additional literature for an in-depth treatment. Consequently, some derivations and proofs in this paper may appear trivial to readers well versed in convex analysis and optimization.

\subsection{Notation}
Scalars are denoted by lowercase latin or greek letters, e.g., $a, \alpha$. Vectors are denoted by lowercase boldface letters, e.g. $\va$, and matrices by uppercase boldface letters, e.g. $\vA$. Vector or matrix components are denoted like scalars, e.g. $v_i, A_{ij}$.

\subsection{Extended real-valued functions}
Optimization problems are occasionally described using extended real-valued functions, which are functions that can take any real value, in addition to the infinite values $-\infty$ and $\infty$. An extended real-value function $\phi: \reals^d \to [-\infty, \infty]$ has an associated effective domain $\dom(\phi) = \{ \vx \in \reals^d: \phi(\vx) < \infty\}$, and it's natural to use such functions to encode constrained optimization problems, where $\dom(\phi)$ or $\dom(-\phi)$ are used to encode constraints of minimization or maximization problems. Throughout this paper, we implicitly assume that any extended real-valued function $\phi$ is:
\begin{itemize}
    \item \emph{proper} - its $-\infty$ nowhere, and $\dom(\phi) \neq \emptyset$, and
    \item \emph{closed} - it's epigraph $\operatorname{epi}(f) = \left\{ (\vx, t) \in \reals^d\times \reals : \phi(\vx) \leq t \right\}$ is a closed set.
\end{itemize}
When $\phi$ is used in the context of a maximization problem, we make those assumptions about $-\phi$.
Any reasonable function of practical interest is proper, and the vast majority of functions useful in practice are closed as well. An extensive introduction to extended real-valued function can be found, for example, in  \cite[chap.~2]{beck_fom}. 

\subsection{Our contributions}
In this work, we consider the following families of functions $f$ when computing $\prox_{\eta f}$. For each family of functions we show examples of machine learning models for whose training it might be useful, devise an algorithm for computing its proximal operator, provide a Python reference implementation, and experimentally measure its efficiency. The families are described below.
\begin{description}
\item[A convex onto linear composition:] 
\begin{equation*}\label{eq:cvx_lin}\tag{CL}
    f(\vx) = h(\va^T \vx + b),
\end{equation*}
where $h: \reals \to \extreals$ is a convex extended real-valued function.
\item[A regularized convex onto linear composition:]
\begin{equation*}\label{eq:cvx_lin_reg}\tag{RCL}
    f(\vx) = h(\va^T \vx + b) + r(\vx),
\end{equation*}
where $h: \reals \to \extreals$, and $r: \reals^d \to \extreals$ are convex extended real-valued functions.
\item[A mini-batch of convex onto linear compositions:]
\begin{equation*}\label{eq:cvx_lin_minibatch}\tag{CL-B}
    f(\vx) = \frac{1}{m} \sum_{i=1}^m h(\va_i^T \vx + b_i),
\end{equation*}
where $h: \reals \to \extreals$ is a convex extended real-valued function, and $m$ is small.
\end{description}
For the above families we develop a framework for computing the proximal operator, based on convex duality, to achieve both algorithmic efficiency and software modularity. We provide a pure Python reference implementation that aims to be moderately efficient. For some problems it may be on par with SGD, whereas for others it may be up to a few dozen times slower. Indeed, a pure Python reference implementation is is not necessarily the most efficient one, and a better alternative in terms of efficiency is a combination of C and Python, where C is used for loop-intensive tasks, such as one-dimensional root finding. We chose PyTorch as our array module, due to easier integration with auto-grad based model training code. The code we build in this paper is available on GitHub at \texttt{https://github.com/alexshtf/inc\_prox\_pt/}\cite{alex_shtof_2024_11500107}.

The remainder of the paper is organized as follows. After discussing previous work below, we develop the initial version for our algorithmic framework for convex onto linear compositions in Section \ref{sec:cvx_lin}. In that section, we also provide full code inline to let the readers appreciate how software modularity is facilitated by our framework. Then, we proceed to extending our framework to regularized convex onto linear compositions in Section \ref{sec:cvx_lin_reg}, and to a mini-batch of convex onto linear compositions in Section \ref{sec:cvx_lin_mb}. These sections contain only the framework code, whereas the remaining code stemming from these sections is available both on our GitHub repository, and in Appendices \ref{app:cvx_lin_reg_code} and \ref{app:cvx_lin_mb_code} in this paper.

\subsection{Previous work}
There has been a significant body of research into the analysis of incremental proximal iteration algorithms in various settings and approximating models, which in addition to the examples we provided above also include \cite{duchi_ruan_feng, asi_duchi_better_models, dinh_stochastic_gauss_newton, davis_weakly_convex, davis_higher_order_growth,asi_duchi_structures_geometry,pointSaga,prox2saga,bsPointSaga}. However, to the best of our knowledge, finding generic algorithms for efficient implementation of incremental proximal algorithms received little attention, and the issue has been partially addressed in the papers focused on the analysis. 

Exceptions to the above rule are methods based on the proximal gradient approach, original proposed by \cite{proxgrad_orig}, where the approximating function is of the form:
\[
f(\vx) = \va^T \vx + r(\vx),
\]
where $r(x)$ is some 'simple' function, usually a regularizer, for which a closed form solution for the proximal operator of $r$ is known.  See, for example, the works of \cite{boyd_parikh, beck_fom} and references therein for examples. In this work we consider a significantly broader family of functions, aimed at machine learning applications, by building on existing theory in convex analysis, and functions whose proximal operator is known.

Proximal operators of functions belonging to the \eqref{eq:cvx_lin} family is covered to some extent in the literature. For example, \cite{kulis2010implicit} shows an explicit formula for the case when the 'outer' function $h$ is the $\ell_2$ cost $h(t)=\frac{1}{2}t^2$. The case of $h(t)=\max(t, 0)$ has an explicit formula in the work of \cite{aprox}, while $h(t)=|t|$ is treated in \cite{davis_weakly_convex}. Additional formulas can be found in \cite{prox2saga}. In \cite{pointSaga} the authors provide an efficient Cython implementation for the Logistic and Hinge loss, but not in a generic framework. While \cite{ryu2014stochastic} show an explicit method for the entire family, they do not show how the method is derived, and provide neither concrete examples, nor code. We treat this family using a uniform framework based on duality, show its benefits from a software engineering perspective of decoupling concerns, and provide explicit code for examples which are useful in machine learning.

The work of \cite{ryu2014stochastic} briefly discusses an algorithm for a simplified version \eqref{eq:cvx_lin_reg} case where the regularizer is assumed to be separable. Their idea is similar to ours, but we devise an algorithm for a more general case, and provide explicit examples and code which is useful for machine learning practitioners. Moreover, they provide a 'switch to SGD heuristic' for improving computational efficiency by switching to a regular stochastic gradient method. This heuristic is orthogonal to the contributions of our paper.

Finally, the \eqref{eq:cvx_lin_minibatch} family is tackled in \cite{aprox_minibatch} for the special case of $h(t)=\max(0, t)$. We devise a uniform framework for a broader family which covers a wide variety of functions $h$, and give several examples which may be useful to practitioners.

\section{A convex onto linear composition}\label{sec:cvx_lin}
In this section we first present some applications of the convex onto linear family described in \eqref{eq:cvx_lin}, and then describe a generic framework for designing and implementing algorithms for computing the proximal point of this function family.

\subsection{Applications}\label{sec:cvx_lin_app}
Compositions of convex onto linear functions appear in a wide variety of classical machine learning problems, but also appear to be useful for other applications as well. This family is useful when training on \emph{one} training sample in each iteration.

\paragraph{Linear regression}
In the simplest case, least-squares regression, we aim to minimize a sum, or an expectation, of costs of the form
\[
f(\vx) = \half ( \va^T \vx + b )^2,
\]
where $(\va, b) \in \reals^n \times \reals$ are the data of a training sample, and $\vx$ is the model parameters vector. In this case we have $h(t)=\half t^2$. Different regression variants are obtained by choosing a different function $h$, for example, using $h(t)=|t|$ we obtain robust regression:
\[
f(\vx) = | \va^T \vx + b |,
\]
and using the $h(t)=\max((p - 1) t, p t)$ for some $p \in (0, 1)$, we obtain linear \emph{quantile regression}.

\paragraph{Logistic regression}\label{par:logreg} For the binary classification problem, we are given an input features $\va \in \reals^n$ and a label $y \in \{-1, 1\}$. The probability of $y = 1$ is modeled using
\[
\sigma(\vx) = \frac{1}{1 + \exp(-\va^T \vx)},
\]
where $\vx$ is the model's parameter vector. The cost incurred by each training sample is computed using the \emph{binary cross-entropy loss}:
\[
f(\vx) = \begin{cases}
-\ln(\sigma(\vx)) & y = -1 \\
-\ln(1 - \sigma(\vx)) & y = 1,
\end{cases}
\]
which after some algebraic manipulation can be written as:
\[
f(\vx) = \ln(1 + \exp(-y \va^T \vx)).
\]
Defining $h(t) = \ln(1 + \exp(t))$, we obtain the convex-onto-linear form.

\paragraph{The APROX model \cite{aprox}}
A cost function $\phi$ which is bounded below, where w.l.o.g we assume that it is bounded below by zero, is approximated by:
\[
f(\vx) = \max(\phi(\vx_{t}) + \langle \nabla \phi(\vx_t), \vx - \vx_t \rangle, 0).
\]
It's similar to a linear approximation, but it also incorporates knowledge about the function's lower bound: if a cost function is non-negative, its approximation should also be. Taking $h(t)=\max(t, 0)$, $\va = \nabla \phi(\vx_t)$, and $b = \phi(\vx_t) - \langle \phi(\vx_t), \vx_t \rangle$, we obtain the convex-onto-linear form \eqref{eq:cvx_lin}.

\paragraph{The prox-linear approximation}
Assume we are given a cost function of the form
\[
\phi(\vx) = h(g(\vx)),
\]
where the \emph{outer} function $h$ is convex, and the \emph{inner} function $g$ is an arbirary function, such as a deep neural network. For instance, suppose we're solving a classification problem using a neural network $u$, whose output is fed to the sigmoid, and then to the binary cross-entropy loss. Similarly to the logistic regression setup above, given a label $y \in \{-1, 1\}$, and an input $\vv{w}$, we obtain:
\[
\phi(\vx) = \ln(1 + \exp(-y u(\vx, \vv{w})).
\]
Taking $h(t) = \ln(1+\exp(t))$ and $g(\vx) = -y u(\vx, \vv{w})$ we have the
desired form.

The idea of the prox-linear method is to linearly approximate the inner function around the current iterate, while leaving the outer function as is. Thus, we obtain the following approximation of the cost $\phi$
\[
f(\vx) = h(g(\vx_t) + \langle \nabla g(\vx_t), \vx - \vx_t \rangle),
\]
and taking $\va = \nabla g(\vx)$, and $b = g(\vx_t) - \langle \nabla g(\vx_t), \vx_t\rangle$ we obtain the desired convex-onto-linear form \eqref{eq:cvx_lin}. See the recent work of \cite{drusvyatskiy_proxlinear} and references therein for extensive anaysis and origins of the method.

\subsection{The proximal operator}
Having shown a variety of potential applications, let's explore the proximal operator of convex-onto-linear functions. Its computation amounts to solving the following problem:
\begin{equation}\label{eq:prox_op_cvx_lin}
\min_{\vx} \quad 
h(\va^T \vx + b) + \frac{1}{2\eta} \| \vx - \vx_t \|_2^2.
\end{equation}
We'll tackle this problem, and most of the remaining cost families, using the well known \emph{convex duality} framework, which we briefly introduce here for completeness. An extensive introduction can be found in many optimization textbooks, such as \cite{beck_fom}.

\subsubsection{Convex duality}
We make a brief introduction to a subset of convex duality theory for unfamiliar readers who would like to understand how we built our framework and how to extend it for their own purposes. Since we don't need most generic convex duality theory this paper, we introduce a simplification. Suppose we're given an optimization problem:
\[\label{eq:lin_constr_prob}\tag{Q}
\min_{\vx} \quad f(\vx) \quad \text{s.t.} \quad \vA \vx = \vb, 
\]
where $\vA \in \reals^m \times \reals^n$ is a matrix, and $f: \reals^n \to \extreals$ is a closed and convex extended real-valued function. Define\[
q(\vv{s}) = \inf_\vx \left\{
\mathcal{L}(\vx, \vs) \defeq f(\vx) + \vs^T (\vA \vx - \vb)
\right\},
\]
namely, we replace the $j^{\text{th}}$ linear constraint by a "price" $s_j$ for its violation, and define $q$ to be the optimal value as a function of these prices. The modified cost function $\mathcal{L}$ is called the \emph{Lagrangian} associated with the constrained problem \eqref{eq:lin_constr_prob}.

First, it's apparent that $q(\vs)$ is \emph{concave}, since it is a minimum of linear functions of $\vs$. Moreover, it's easy to see that $q(\vs)$ is a lower bound for the optimal value of \eqref{eq:lin_constr_prob}, using the simple observation that minimizing over a subset of $\reals^n$ produces a value that is higher or equal to the minimization over the entire space:
\begin{align*}
    q(\vs) 
     &= \inf_\vx \left\{f(\vx) + \vs^T (\vA \vx - \vb) \right\} \\
     &\leq \inf_\vx \left\{f(\vx) + \vs^T (\vA \vx - \vb) : \vA \vx = \vb \right\} \\
     &= \inf_\vx \left\{f(\vx) : \vA \vx = \vb \right\}
\end{align*}
The problem of finding the "best" lower bound is called the \emph{dual problem} associated with \eqref{eq:lin_constr_prob}, namely:
\[\label{eq:lin_constr_dual}\tag{D}
\max_\vs \quad q(\vs) \quad \text{s.t.} \quad q(\vs) > -\infty,
\]
whereas the original problem \eqref{eq:lin_constr_prob} is called the \emph{primal} problem. A well known result in convex analysis is that with slight technical conditions, the optimal values of the primal and the dual problems coincide:
\begin{theorem}[Strong Duality]\label{thm:strong_duality}
Suppose that $f$ is a convex closed extended real-valued function, that the optimal value of \eqref{eq:lin_constr_prob} is finite, namely,
\[
f_\mathrm{opt} = \inf_\vx \left\{f(\vx) : \vA \vx = \vb \right\} > -\infty,
\]
and that there exists some feasible solution $\hat{\vx}$. Then,
\begin{enumerate}[label=(\roman*)]
    \item The optimal value of the dual problem \eqref{eq:lin_constr_dual} is attained at some optimal solution $\vs^*$, and it is equal $f_\mathrm{opt}$.
    \item The optimal solutions of \eqref{eq:lin_constr_prob} are $\vx^* \in \argmin_\vx \mathcal{L}(\vx, \vs^*)$ which are feasible. In particular, if $\mathcal{L}(\vx, \vs^*)$ has a unique minimizer, then it must be an optimal solution of \eqref{eq:lin_constr_prob}.
\end{enumerate}
\end{theorem}
\begin{proof}
(i) is a special case of Theorem A.1 and (ii) is a special case of Theorem (ii) in \cite{beck_fom}.
\end{proof}
The strong duality theorem has an important consequence for the case when the dual problem \eqref{eq:lin_constr_dual} significantly easier to solve than the primal problem \eqref{eq:lin_constr_prob}. Having obtained its optimal solution $\vs^*$, we can recover the optimal solution of the primal problem by minimizing the Lagrangian function $\mathcal{L}(\vx, \vs^*)$ over $\vx$.

\subsubsection{Employing duality}
Duality requires a constrained optimization problem, whereas the proximal operator in Equation \eqref{eq:prox_op_cvx_lin} aims to solve an unconstrained problem. However, constraints are easily added with the help of an auxiliary variable, and we can equivalently solve:
\[
\min_{\vx, z} \quad h(z) + \frac{1}{2\eta} \|\vx - \vx_t \|_2^2 \quad \text{s.t.} \quad z = \va^T \vx + b
\]
The dual objective is therefore:
\begin{align*}
    q(s) &= \inf_{\vx, z} \left \{
        h(z) + \frac{1}{2\eta} \|\vx - \vx_t \|_2^2 + s(\va^T \vx + b - z)
    \right\} \\
    &= \underbrace{ \min_{\vx} \left\{ \frac{1}{2\eta} \|\vx - \vx_t \|_2^2 + s\va^T \vx  \right\} }_{A} + \underbrace{ 
    \inf_z \{ h(z) - s z \}}_{B} + s b 
\end{align*}
The term denoted by $A$ is a strictly convex quadratic function, whose minimizer is
\begin{equation}\label{eq:cvx_lin_recover}
    \vx^* = \vx_t - \eta s \va,
\end{equation}
and the minimum itself is
\[
A = -\frac{\eta \| \va \|_2^2}{2} s^2 + (\va^T \vx_t) s.
\]
The term denoted by B can be alternatively written as
\[
B = -\sup_z \left\{ s z - h(z) \right\} = -h^*(s),
\]
where $h^*$ is a well-known object in optimization called the \emph{convex conjugate} of $h$. A catalogue of pairs of convex conjugate functions is available in a variety of standard textbooks on optimization, e.g. \cite{beck_fom}. For completeness, in Table \ref{tab:conjugates} we show conjugate pairs which are useful for the machine-learning oriented examples in this paper.
\begin{table}[hbtb]
    \centering
    {\renewcommand{\arraystretch}{1.3}\begin{tabular}{|c|c|c|c|p{.25\textwidth}|}
        \hline
        $h(z)$ & $\dom(h)$ & $h^*(s)$ & $\dom(h^*)$ & Useful for \\
        \Xhline{4\arrayrulewidth}
        $\frac{1}{2} z^2$ & $\reals$ &  $\frac{1}{2} s^2$ & $\reals$ &  Linear least squares \\
        \hline
        $\ln(1+\exp(z))$ & $\reals$ & $\begin{gathered}s \ln(s) + (1 - s) \ln(1 - s) \\ \text{ where } 0 \ln(0) \equiv 0\end{gathered}$ & $[0,1]$ & Logistic regression \\
        \hline
        $\max(z, 0)$ & $\reals$ & $\begin{cases} 0 & s \in [0, 1] \\ \infty & \text{else} \end{cases} $ & $[0,1]$ & AProx model, Hinge loss \\
        \hline
    \end{tabular}}
    \caption{Example convex conjugate pairs}
    \label{tab:conjugates}
\end{table}
Convex conjugates possess two important properties. First, under mild technical conditions, we have $(h^*)^* = h$, i.e. the bi-conjugate of a convex function is the function itself. These conditions hold for most functions we care about in practice, including the functions in Table \ref{tab:conjugates}. Second, the conjugate is always convex. 

Summarizing the above, the dual problem aims to solve the following \emph{one dimensional} and \emph{strongly concave} maximization problem:
\begin{equation}
    \label{eq:cvx_lin_dual}
    \max_s \quad q(s) \equiv -\underbrace{\frac{\eta \| \va \|_2^2}{2}}_{\frac{\alpha}{2}} s^2 + \underbrace{(\va^T \vx_t + b)}_{\beta} s - h^*(s).
\end{equation}
Strong concavity implies the existance of a unique maximizer $s^*$, while the strong duality theorem implies we can recover the proximal operator we seek by substituting the above maximizer into equation \eqref{eq:cvx_lin_recover}. 

We implement the above idea the \texttt{IncConvexOnLinear} class below using the \texttt{PyTorch} library.
\begin{pycode}
import torch

class IncConvexOnLinear:
    def __init__(self, x, h):
        self._h = h
        self._x = x
        
    def step(self, eta, a, b):
        """
        Performs the optimizer's step, and returns the loss incurred.
        """
        h = self._h
        x = self._x
        
        # compute the dual problem's coefficients
        alpha = eta * torch.sum(a**2)
        beta = torch.dot(a, x) + b
        
        # solve the dual problem
        s_star = h.solve_dual(alpha.item(), beta.item())
        
        # update x
        x.sub_(eta * s_star * a)
        
        return h.eval(beta.item())
    
    def x(self):
        return self._x
\end{pycode}
Note, that from a software engineering perspective, we encode the function $h$ using an object which has two operations: compute the value of $h$, and solve the dual problem. In the next sub-sections we will implement three such objects: \texttt{HalfSquared} for $h(z) = \frac{1}{2} z^2$, \texttt{Logistic} for $h(z) = \ln(1+\exp(z))$, and \texttt{Hinge} for $h(z) = \max(z, 0)$. 

As an example, applying the implicit online learning idea of \cite{kulis2010implicit} to the linear least squares problem looks like this:
\begin{pycode}
x = torch.zeros(d)
optimizer = IncConvexOnLinear(x, HalfSquared())
for t, (a, b) in enumerate(my_data_set):
    eta = get_step_size(t)
    optimizer.step(eta, a, b) 

print('The parameters are: ' + str(x))
\end{pycode}

\subsubsection{The half-squared function}
For  $h(z) = \frac{1}{2} z^2$, according to Table \ref{tab:conjugates}, we have $h^*(s) = \frac{1}{2}s^2$, and thus the dual problem in Equation \eqref{eq:cvx_lin_dual} amounts to maximizing
\[
q(s) = -\frac{\alpha}{2} s^2 + \beta s - \frac{1}{2} s^2 = -\frac{1 + \alpha}{2} s^2 + \beta s.
\]
Hence, in this case $q(s)$ is a simple concave parabola, maximized at
\[
s^* = \frac{\beta}{1 + \alpha}.
\]
Consequently, our \texttt{HalfSquared} class is:
\begin{pycode}
import torch
import math

class HalfSquared:
    def solve_dual(self, alpha, beta):
        return beta / ( 1 + alpha)
        
    def eval(self, z):
        return 0.5 * (z ** 2)
\end{pycode}

\subsubsection{The logistic function}\label{sec:cvx_lin_logistic_class}
For $h(z) = \ln(1 + \exp(z))$, according to Table \ref{tab:conjugates}, the dual problem in  Equation \eqref{eq:cvx_lin_dual} amounts to maximizing
\begin{equation}\label{eq:logistic_dual}
q(s) = -\frac{\alpha}{2} s^2 + \beta s - s \ln(s) - (1 - s) \ln(1 - s).
\end{equation}
The following simple result paves the way towards maximizing $q$.
\begin{proposition}\label{prop:logistic_dual_unique}
The function $q$ defined in Equation \eqref{eq:logistic_dual} has a unique maximizer inside the open interval $(0, 1)$.
\end{proposition}
\begin{proof}
Note that $\dom(q) = [0, 1]$, thus maximizing $q$ is done over a compact interval. Its continuity with the Weirstrass theorem ensures that it has a maximizer in $[0, 1]$, and its strict concavity ensures that the maximizer is unique.  The derivative $q'(s) = -\alpha s + \beta - \ln(s) + \ln(1+s)$ is continuous, decreasing, and satisfies:
\[
\lim_{s \to 0} q'(s) = \infty, \quad \lim_{s \to 1} q'(s) = -\infty.
\]
Hence, there must be a unique point in the open interval $(0, 1)$ where $q'(s) = 0$. Since $q$ is concave, that point must be the maximizer.

$\square$
\end{proof}
Proposition \ref{eq:prox_op_cvx_lin} implies that we can maximize $q$ by employing any root finding algorithm to find a zero of its derivative. Its initial interval $[l, u]$ can be found by:
\begin{itemize}
    \item $l = 2^{-k}$ for the smallest positive integer $k$ such that $q'(2^{-k}) > 0$. 
    \item $u = 1 - 2^{-k}$ for the smallest positive integer $k$ such that $q'(1 - 2^{-k}) < 0$.
\end{itemize}
For our reference implementation, we chose to rely on Brent's method \cite[Chapter~5]{brent2013algorithms} readily available in the \texttt{scipy} package. The method finds a $\varepsilon$-approximate root in $O(\ln(\tfrac{1}{\varepsilon}))$ iterations in the worst case, but is typically significantly faster than naive bisection. The resulting \texttt{Logistic} class is listed below:
\begin{pycode}
from scipy.optimize import brentq

class Logistic:
    def solve_dual(self, alpha, beta, tol = 1e-16):
        def qprime(s):
            return -alpha * s + beta + math.log(1-s) - math.log(s)
        
        # compute [l,u] containing a point with zero qprime
        l = 0.5
        while qprime(l) <= 0:
            l /= 2
        
        u = 0.5
        while qprime(1 - u) >= 0:
            u /= 2
        u = 1 - u
        
        solution = brentq(qprime, l, u)
        return solution
   
    def eval(self, z):
        return math.log(1 + math.exp(z))
\end{pycode}

\subsubsection{The Hinge function}
For $h(t) = \max(0, t)$, according to Table \ref{tab:conjugates}, the dual problem in Equation \eqref{eq:cvx_lin_dual} amounts to solving the following constrained problem:
\[
\max_s \quad q(s) = -\frac{\alpha}{2} s^2 + \beta s \quad \text{s.t.} \quad s \in [0, 1].
\]
The above is a maximum of a concave parabola over an interval, and it's easy to see that its maximizer is
\[
s^* = \max\Bigl(0, \min\Bigl(1, \frac{\beta}{\alpha}\Bigr)\Bigr).
\]
The resulting \texttt{Hinge} class is therefore:
\begin{pycode}
class Hinge:
    def solve_dual(self, alpha, beta):
        return max(0, min(1, beta / alpha))
        
    def eval(self, z):
        return max(0, z)
\end{pycode}

\subsection{Empirical evaluation}
To evaluate the efficiency of our implementation, we compare it against a regular incremental gradient method. Our algorithms in this section are meant to work on one training sample at a time, but incremental gradient methods are usually used with mini-batches of samples. Thus, we'd like to see how competitive is our implementation against incremental gradient methods with various mini-batch sizes, including a mini-batch size of one sample, to appreciate the usefulness of our implementation as a tool by the research community.

We applied our evaluation on the Logistic regression and Least Squares problems, one uses the \texttt{Logistic} class whereas the other uses the \texttt{HalfSquared} class. We measured the total time to make one epoch over data-sets of various dimensions and of various lengths, and plotted a regression line, whose slope allows us to convince ourselves than our implementation is slower than a regular incremental gradient method by a modest constant factor. The results are plotted in Figure \ref{fig:cvx_lin_empirical_speed}. To summarize, the incremental proximal point method we devised here is roughly 2-8 times slower, per iteration, than an incremental gradient method when mini-batches are used, but faster than an incremental gradient method without mini-batching. It's also interesting to point out that without mini-batching, our implementation is \emph{faster} than PyTorch, mainly due to an overhead of its AutoGrad mechanism with small mini-batches. Therefore, we do not believe it to be a ``fair'' comparison, but merely a part of the demonstration that our implementation is good enough to be useful.

\begin{figure}
    \centering
    \includegraphics[width=\textwidth]{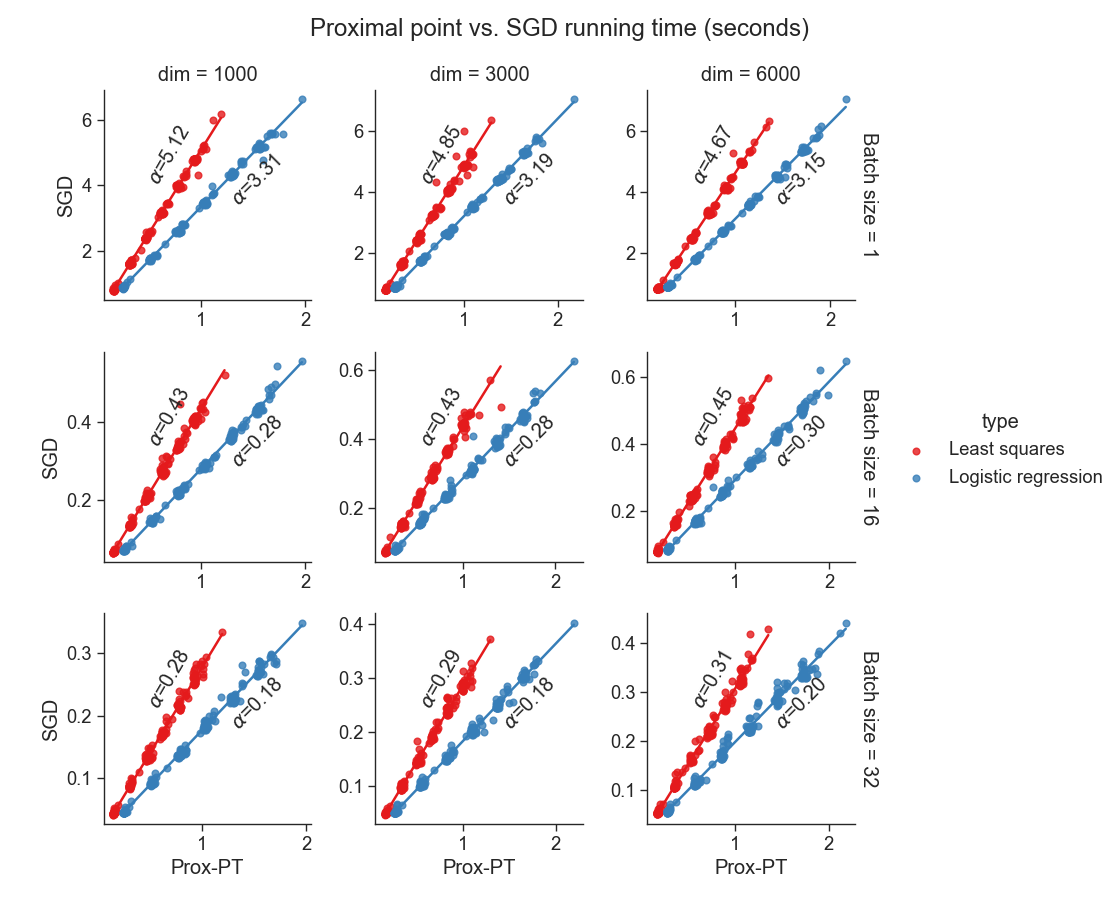}
    \caption{Execution speed evaluation of incremental proximal point. Each point is a timing of a pair of experiments on the same problem, where the $x$ coordinate is the execution time of one SGD epoch, whereas the $y$ coordinate is the execution time of one proximal point epoch. The corresponding line is a linear regression line, with its slope $\alpha$ labeled, to appreciate the ratio between the SGD and proximal point execution times, on average. The columns are various problem dimensions, from 1000 to 6000, and the rows are various mini-batch sizes for the incremental gradient method. We can see by the first row, for example, that without mini-batching the proximal-point method is actually faster than an incremental gradient method  ($\alpha > 0$) based on PyTorch's automatic differentiation. In the last row, for example, we can see that for batch sizes of 32 samples the incremental proximal point method is roughly 4 times slower for least-squares problems and 6 times slower for logistic regression problems compared to an incremental gradient method.}
    \label{fig:cvx_lin_empirical_speed}
\end{figure}

To convince ourselves that our implementation is indeed correct, we reproduce the results seen in \cite{asi_duchi_better_models} which show that proximal-point methods are much more robust to step-size choice. For randomly generated logistic regression and least-squares problems, we solve both problems using a range of step-sizes by performing one epoch over a random shuffle of the data, and measure the training loss of that epoch. We repeat the experiment for each step-size 30 times to obtain a confidence band around the results. The problems are of dimension 100, and the data-set size is 100,000. So ideally, one epoch should simulate a large random sample from a stationary distribution. The results are plotted in Figure \ref{fig:cvxlin_stability_reproduction}. Indeed, the results look similar to what the authors of \cite{asi_duchi_better_models} have obtained.

\begin{figure}
    \centering
    \includegraphics[width=\textwidth]{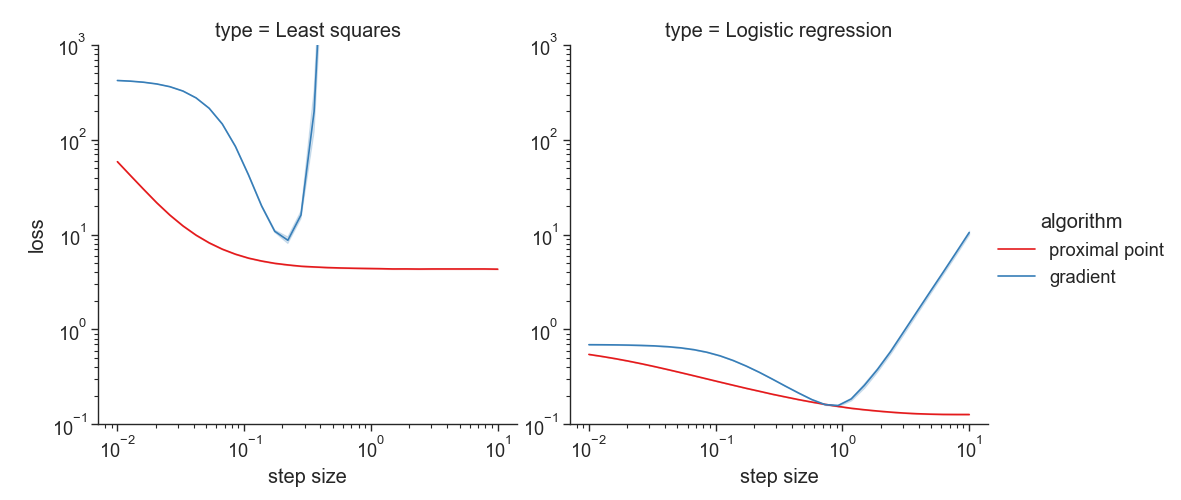}
    \caption{Reproduction of the stability results of \cite{asi_duchi_better_models}. Indeed, for both logistic regression and least squares problems, the training loss of an incremental proximal-point method is significantly more stable w.r.t the step-size choice, reinforcing the claim that with an incremental proximal-point method we can just "make an few educated guesses" instead of performing an extensive hyper-parameter search for the best step-size.}
    \label{fig:cvxlin_stability_reproduction}
\end{figure}

\subsection{Summary}
Typical optimizers for machine learning rely on a well-established tool from analysis - the gradient. Using duality, we were able to design an optimizer for the convex onto linear setup using another well established tool - the convex conjugate. Moreover, duality allowed us to decouple the complexity associated with the convex function $h$ into a separate class whose only purpose is using the conjugate to solve a one dimensional optimization problem. Thus the framework is extensible with additional functions $h$, and designing a class for such $h$ comprises of two steps: obtaining the conjugate $h^*$, preferably from a textbook, and figuring out how to solve the resulting dual problem.

\section{Regularized convex onto linear composition}\label{sec:cvx_lin_reg}
Often, machine learning models are trained by using a regularized loss function. Thus, the applications of the regularized convex-onto-linear model \eqref{eq:cvx_lin_reg} are identical to the described in Section \ref{sec:cvx_lin_app}, e.g. regularized linear least squares, or regularized logistic regression. It's also interesting to note that, using $h(z)=\max(0, z)$ and $r(\vx) = \mu \| \vx \|_2^2$ we can also formulate the problem of training an SVM \cite{cortes1995support}. Since the applications are obvious, we dive directly into the computation of the proximal operator, by solving
\begin{equation}\label{eq:cvx_lin_reg_prox}
    \min_{\vx} \quad h(\va^T \vx + b) + r(\vx) + \frac{1}{2\eta} \|\vx - \vx_t \|_2^2,
\end{equation}
where $h: \reals \to \extreals$, $r: \reals^d \to \extreals$, are convex extended real-valued functions. We also assume that the regularizer $r$ is "simple", meaning that we can efficiently compute its proximal operator $\prox_r(\vx)$.  Explicit formulae for the proximal operators can be found in a variety of textbooks on optimization, e.g. \cite[chap.~6]{beck_fom}. Examples include the squared Euclidean norm $r(\vx) = \frac{\mu}{2} \|\vx\|_2^2$, the Euclidean norm $r(\vx) = \mu \| \vx \|_2$, or the $\ell_1$ norm $r(\vx) = \mu \| \vx \|_1$. For completeness, we include a short table of proximal operators of the above-mentioned functions in Table \ref{tab:prox_operators}.

\begin{table}[htbp]
    \centering
    \begin{tabular}{|c|c|p{.5\textwidth}|}
    \hline
    $r(\vx)$ & $\prox_{\mu r}(\vx)$ & Remarks \\
    \Xhline{4\arrayrulewidth}
    $\|\vx\|_1$ & $\max(|\vx| - \mu \mathbf{1}, \mathbf{0}) \cdot \sgn(\vx)$ & Absolute value, $\max$, $\sgn$, and $\cdot$ are component-wise operations. Available as the \texttt{softshrink} function in PyTorch. \\
    \hline
    $\half \|\vx\|_2^2$ & $\frac{1}{1+\mu} \vx$ & \\
    \hline
    $\|\vx\|_2$ & $\left(1 - \frac{\mu}{\max(\mu, \|\vx\|_2)} \right) \vx$ & \\
    \hline
    \end{tabular}
    \caption{Proximal operators of commonly used regularizers}
    \label{tab:prox_operators}
\end{table}

\subsection{A dual problem}
Following the path paved in Section \ref{sec:cvx_lin}, we begin from deriving a dual to the problem in Equation \eqref{eq:cvx_lin_reg_prox}. Adding the auxiliary variable $z = \va^T \vx + b$, we obtain the equivalent optimization problem
\[
\min_{\vx, z} \quad h(z) + r(\vx) + \frac{1}{2\eta} \|\vx - \vx_t\|_2^2 \quad \text{s.t.} \quad z = \va^T \vx + b
\]
Minimizing the Lagrangian, we obtain:
\begin{equation}\label{eq:cvx_lin_reg_dual}
\begin{aligned}
q(s) &= \inf_{\vx, z} \left\{ \mathcal{L}(\vx, z, s) \equiv h(z) + r(\vx) + \frac{1}{2\eta} \|\vx - \vx_t\|_2^2 + s(\va^T \vx + b - z) \right\} \\
 &= \inf_{\vx} \biggl\{ r(\vx) + \frac{1}{2\eta} \|\vx - \vx_t\|_2^2 + s \va^T \vx \biggr\} + \inf_z \{ h(z) - s z \} + s b 
\end{aligned}
\end{equation}
The remaining challenge is the computing the first infimum. To that end, we need to introduce another well-known concept in optimization - a close relative of the proximal operator.
\begin{definition}[Moreau Envelope \cite{moreau_envelope}]
Let $\phi: \reals^n \to \extreals$ be a convex extended real-valued function. The Moreau envelope of $\phi$ with parameter $\eta$, denoted by $M_\eta \phi$, is the function:
\[
M_\eta \phi(\vx) = \min_\vu \left\{ \phi(\vu) + \frac{1}{2\eta} \|\vu - \vx\|_2^2 \right\}.
\]
\end{definition}
Since the proximal operator is the minimizer of the minimum in the definition above, which is Moreau proved that is always attained, an alternative way to write the Moreau envelope of a function is obtained by replacing $\vu = \prox_{\eta \phi}(\vx)$ inside the minimization objective above:
\begin{equation}\label{eq:envelope_prox_formula}
M_\eta \phi(\vx) = \phi(\prox_{\eta \phi}(\vx)) + \frac{1}{2\eta} \|\prox_{\eta \phi}(\vx) - \vx  \|_2^2.
\end{equation}
Thus, whenever we have an explicit formula of the proximal operator, we also have an explicit formula of the Moreau envelope. How does it help us with our challenge of minimizing $\mathcal{Q}$ over $\vx$? The following proposition provides the answer.
\begin{proposition}\label{prop:min_Q}
Let $q(s)$ and the Lagrangian $\mathcal{L}(\vx, z, s)$ be as defined in Equation \eqref{eq:cvx_lin_reg_dual}. Then,
\[
q(s) = M_{\eta} r(\vx_t - \eta s \va) + (\va^T \vx_t + b) s - \frac{\eta \|\va\|_2^2}{2} s^2 - h^*(s).
\]
Moreover, the unique minimizer of the Lagrangian $\mathcal{L}$ w.r.t $\vx$ is 
\[
\vx^* = \prox_{\eta r}(\vx_t - \eta s \va).
\]
\end{proposition}
\begin{proof}
Recall, that for any $\vx, \vy$ we can open the squared Euclidean norm using the formula
\[
\half \|\vx + \vy\|_2^2 = \half \| \vx \|_2^2 + \vx^T \vy + \half \| \vy \|_2^2,
\]
and re-arranging the above leads to the \emph{square completion} formula
\[
\frac{1}{2}\|\vx\|_2^2+ \vx^T \vy = \frac{1}{2}\|\vx + \vy\|_2^2 - \frac{1}{2}\|\vy\|_2^2.
\]
Using the above two formulas, we compute:
\begin{align*}
r(\vx) &+\frac{1}{2\eta} \|\vx-\vx_t\|_2^2 + s \va^T \vx \\
 &= \frac{1}{\eta} \left[ \eta r(\vx) +  \frac{1}{2} \|\vx - \vx_t\|_2^2 + \eta s \va^T \vx \right] & \leftarrow \text{Factoring out } \frac{1}{\eta} \\
 &= \frac{1}{\eta} \left[ \eta r(\vx) + {\color{black!50}\frac{1}{2} \|\vx\|_2^2 - (\vx_t - \eta s \va )^T \vx} + \frac{1}{2} \|\vx_t\|_2^2 \right] & \leftarrow \text{opening } \frac{1}{2}\|\vx-\vx_t\|_2^2 \\
 &= \frac{1}{\eta} \left[ \eta  r(\vx) + {\color{black!50}\frac{1}{2} \|\vx - (\vx_t - \eta s \va)\|_2^2 - \frac{1}{2} \|\vx_t - \eta s \va\|_2^2 } + \frac{1}{2} \|\vx_t\|_2^2 \right] & \leftarrow \text{square completion}\\
 &= \left[ r(\vx)+\frac{1}{2\eta} \|\vx - (\vx_t - \eta s \va)\|_2^2  \right] - \frac{1}{2\eta} \|\vx_t - \eta s \va\|_2^2 + \frac{1}{2\eta} \|\vx_t\|_2^2  &\leftarrow{\text{Multiplying by }\frac{1}{\eta}} \\
 &= \left[ r(\vx)+\frac{1}{2\eta} \|\vx - (\vx_t - \eta s \va)\|_2^2  \right] + (\va^T \vx_t) s - \frac{\eta \|\va\|_2^2}{2} s^2 &\leftarrow{\text{re-arranging}}
\end{align*}
Plugging the above expression into the formula of $q(s)$, we obtain:
\begin{align*}
q(s) 
 &= \inf_{\vx} \left \{ r(\vx)+\frac{1}{2\eta} \|\vx - (\vx_t - \eta s \va)\|_2^2 \right\} + (\va^T \vx_t + b) s - \frac{\eta \|\va\|_2^2}{2} s^2 - h^*(s) \\
 &= M_\eta r(\vx_t - \eta s \va) + (\va^T \vx_t + b) s - \frac{\eta \|\va\|_2^2}{2} s^2 - h^*(s)
\end{align*}
Moreover, since the infimum over $\vx$ above is attained, we can replace it with a minimum, and by definition the minimizer is
\[
\vx^* = \prox_{\eta r}(\vx_t - \eta s \va)
\]
$\square$
\end{proof}

The significance of Proposition \ref{prop:min_Q} is due to the fact that we can design an algorithm for computing the proximal operator of regularized convex onto linear losses using three textbook concepts: the Moreau envelope of $r$, the convex conjugate of $h$, and the proximal operator of $r$. The only thing we need to manually derive ourselves is a way to maximize the dual objective $q$. The basic method, directly applying Proposition \ref{prop:min_Q} and the strong duality theorem (Theorem \ref{thm:strong_duality}) consists of the following three steps:
\begin{itemize}
    \item Form $q(s) = M_{\eta} r(\vx_t - \eta s \va) + (\va^T \vx_t + b) s - \frac{\eta \|\va\|_2^2}{2} s^2 - h^*(s)$.
    \item Solve the dual problem: find a maximizer $s^*$ of $q(s)$
    \item Compute $\prox_f(\vx_t) = \prox_{\eta r}(\vx_t - \eta s^* \va)$
\end{itemize}

\subsection{Computing the proximal operator}\label{sec:cvx_lin_reg_prox_impl}
An important aspect of solving the dual problem is figuring out $\dom(-q)$, since this set encodes the constraints of the dual problem. It turns out that Moreau envelopes are finite everywhere \cite[Theorem~6.55]{beck_fom}, and thus $\dom(-q) = \dom(h^*)$. Moreover, the strong duality theorem (Theorem \ref{thm:strong_duality}) ensures that $q$ attains its maximum, so a maximizer $s^* \in \dom(h^*)$ \emph{must} exist.

The dual problem is \emph{one dimensional}, and there is a variety of reliable algorithms and their implementations for maximizing such functions.  Indeed, this time we will not bother implementing a procedure to maximize $q$, but use a readily available implementations in the \texttt{SciPy} package. When $\dom(-q)$ is a compact interval, we will employ the \texttt{scipy.optimize.fminbound} function on $-q$, which uses Brent's method \cite[Chapter~5]{brent2013algorithms}, and requires us to provide a function, and a compact interval where its maximizer must lie.

When the interval $\dom(-q)$ is not compact, we will have to locate a compact interval which contains a maximizer of $q$. To that end, we will require $h^*$ to be continuously differentiable, strictly convex, and $\dom(h^*)$ to be open. The object representing $h$ needs to provide two sequences $l_1 > l_2 > \dots$ converging to the left endpoint of $\dom(-q)$, and $u_1 < u_2 < \dots$ converging to the right endpoint of $\dom(-q)$. We will shortly see the details, but the general idea of using these sequences is similar to our search for an initial interval in case of the \texttt{Logistic} class we implemented in Section \ref{sec:cvx_lin_logistic_class}, where the sequences were $2^{-1}, 2^{-2}, \dots$, and $1-2^{-1}, 1-2^{-2}, \dots$. Having found an interval containing a maximizer, we will, again, employ the \texttt{scipy.optimize.fminbound} on $-q$ function to find our maximizer.

\paragraph{Remark}
Note, that the Scipy's \texttt{scipy.optimize.minimize\_scalar} function also supports finding the initial \emph{bracket} for Brent's method when we don't possess a compact interval containing a minimizer. However, as of the time of writing of this paper, the \texttt{minimize\_scalar} function is not able to handle half-infinite, such as $[0, \infty)$. It supports either a compact domain, or the entire real line. Thus, to be as generic as possible, we resort to finding the initial interval ourselves using the above-mentioned pairs of sequences.

\paragraph{The case of a compact domain}
To employ Brent's method for finding a maximizer $s^*$ of $q$, we need to be able to evaluate the function $q$. To that end, we require an oracle for evaluating the Moreau envelope of $r$, and the convex conjugate $h^*$. Moreover, to recover the solution of the primal problem, we need an oracle for computing the proximal operator of $r$. 

\paragraph{The case of a non-compact domain}
This case requires us to employ an additional result about Moreau envelopes.
\begin{proposition}\label{prop:envelope_properties}
Let $\phi: \reals^d \to \extreals$ be a convex extended real-valued function, and let $M_\mu \phi$ be its Moreau envelope. Then $M_\mu \phi$ is continuously differentiable with gradient 
\begin{equation}\label{eq:moreau_gradient}
    \nabla M_\mu \phi(\vx) = \frac{1}{\mu}\left( \vx - \prox_{\mu \phi}(\vx) \right)
\end{equation}
\end{proposition}
\begin{proof}
See \cite[Theorem~6.55]{beck_fom}
\end{proof}
Recall, that we require $h^*$ to be continuously differentiable, strictly convex, and $\dom(h^*)$ to be open. Using \eqref{eq:moreau_gradient} and the chain rule to compute the derivative of $q(s)$, and then simplifying, we obtain:
\begin{equation}\label{eq:cvx_lin_reg_qprime}
    q'(s) = \va^T \prox_{\eta r}(\vx_t - \eta s \va) - {h^*}'(s) + b
\end{equation}
For the non-compact domain case, we will require that there is a unique maximizer $s^*$ in the \emph{interior} of $\dom(h)$, which also implies that $q'(s^*) = 0$. The above is satisfied when, for example, when $h^*$ is strictly convex and $\dom(h^{*})$ is open, or when $h^*$ is a \emph{convex function of Legendre type} \cite[sect.~26]{rockafellar_convex_analysis}, which is a condition satisfied by, for example, by $h^*(s) = \frac{1}{2} s^2$ in the least-squares setup. Strict convexity of $h^*$ implies that $q'$ is strictly decreasing, and thus any point $l < s^*$ has a positive derivative, whereas any $u > s^*$ has a negative derivative. Thus, we can find an interval $[l, u]$ containing $s^*$ just by scanning to the left until we find a point whose derivative is positive, and to the right until we find a point whose derivative is negative. And that's exactly why we need the sequences $\{l_k\}_{k=1}^\infty$ and $\{u_k\}_{k=1}^\infty$ for - to tell us how to scan towards the left and right boundaries of $\dom(-q) = \dom(h^*)$. For example, if $\dom(h^*) = [1, \infty)$, we may scan towards the left end-point using the sequence $\ell_k = 1 + 2^{-k}$, and towards the right-endpoint, which is infinity, using the sequence $u_k = 1 + 2^k$.

It is important to note that computing $q'$ requires evaluating the proximal operator of the regularizer $r$. Therefore, each iteration in the search for the optimal $s^*$ \emph{depends on the dimension $d$}, and hence is typically significantly slower than the convex onto linear setting without regularization, that was discussed in Section \ref{sec:cvx_lin}. Of course, it would be possible to construct a dedicated procedure for each combination of outer function $h$ and regularizer $r$, which may occasionally be fast, but it defeats one of the main purposes of our framework - its modularity. 

The fact that the computational complexity of each iteration is linear in the dimension suggests that we could have employed an accelerated first-order method, such as \cite{beck2009fast}, directly on the primal proximal sub-problem. To see why we chose \emph{not} to pursue this direction, note that Brent's method \cite{brent2013algorithms} used by \texttt{fminbound} in our code enjoys \emph{superlinear convergence} (of degree $~1.67$) in the vicinity of the optimum. Accelerated first-order methods have a convergence rate that is linear, at best, when their objective is strongly convex. This suggests that Brent's method, in addition to fitting well into our dual decomposition approach, is also potentially more efficient due to reduced iteration count. We chose to rely on the fact that for one-dimensional problems we have specialized methods that are often faster than methods that work in arbitrary dimensions.

\paragraph{The implementation}
Combining the two cases above, here is an implementation of our optimizer.
\begin{pycode}
from scipy.optimize import fminbound
import torch


class IncRegularizedConvexOnLinear:
    def __init__(self, x, h, r):
        self._x = x
        self._h = h
        self._r = r

    def step(self, eta, a, b):
        x = self._x
        h = self._h
        r = self._r

        if torch.is_tensor(b):
            b = b.item()

        lin_coef = (torch.dot(a, x) + b).item()
        quad_coef = (eta / 2.) * a.square().sum().item()
        loss = h.eval(lin_coef) + r.eval(x).item()

        def qprime(s):
            prox = r.prox(eta, x - eta * s * a)
            return torch.dot(a, prox).item() \
                   - h.conjugate_prime(s) \
                   + b

        def q(s):
            return r.envelope(eta, x - eta * s * a) \
                   + lin_coef * s \
                   - quad_coef * (s ** 2) \
                   - h.conjugate(s)

        if h.conjugate_has_compact_domain():
            l, u = h.domain()
        else:
            # scan left until a positive derivative is found
            l = next(s for s in h.lower_bound_sequence() if qprime(s) > 0)

            # scan right until a negative derivative is found
            u = next(s for s in h.upper_bound_sequence() if qprime(s) < 0)

        min_result = fminbound(lambda s: -q(s), l, u)
        s_prime = min_result.x
        x.set_(r.prox(eta, x - eta * s_prime * a))
        
        return loss
\end{pycode}

Let's look at a usage example. Assuming that $h(z) = \frac{1}{2} z^2$ is represented using the \texttt{HalfSquared} class, and $r(\vx) = \| \vx \|_1$ is represented using the \texttt{L1Reg} class, we can perform an epoch of training an L1 regularized least-squares model (Lasso) with regularization parameter $0.1$ using the following code:
\begin{pycode}
x = torch.zeros(d)

optimizer = IncRegularizedConvexOnLinear(x, HalfSquared(), L1Reg(0.1))
epoch_loss = 0.
for t, (a, b) in enumerate(get_training_data()):
    eta = get_step_size(t)
    epoch_loss += optimizer.step(eta, a, b)
    
print('Model parameters = ', x)
print('Average epoch loss = ', epoch_loss / t)
\end{pycode}

We believe that Section \ref{sec:cvx_lin} demonstrated how to modularity is achieved, and therefore the implementation of the \texttt{HalfSquared}, \texttt{Logistic}, and \texttt{Hinge} classes for the functions $h$ representing various losses, and of the \texttt{L1Reg} and \texttt{L2Reg}, and \texttt{L2NormReg} classes for regularization with $\| \cdot \|_1$, $\| \cdot \|_2^2$, and $\| \cdot \|_2$ are in Appendix \ref{app:cvx_lin_reg_code}.

\subsection{Empirical evaluation}
Before evaluating the efficiency empirically, a word about the computational complexity is in place. Computing the proximal step requires solving a one-dimensional dual problem using an optimization algorithm over the real line. The algorithms employed by the \texttt{fminbound} function typically achieve $\varepsilon$ accuracy in terms of distance to a minimize using $O(\ln(\frac{1}{\varepsilon}))$ function evaluations. By default, SciPy's implementation achieves $\varepsilon=10^{-10}$, meaning that $\ln(\varepsilon^{-1}) \approx 23$. However, this time, evaluating $q(s)$ requires computing the proximal operator of our regularizer, which entails a computational complexity of $O(n)$ for $\vv{x} \in \mathbb{R}^n$. This is in contrast to the convex-onto-linear case we saw in Section \ref{sec:cvx_lin}, where evaluating $q(s)$ was independent of the problem's dimension. The computational cost per iteration is dominated by few dozen evaluations of $\prox_{\eta r}$, and thus we expect it to be a few dozen times slower than a regular proximal-gradient step (without mini-batches).

As was the case with the convex-onto-linear setting, We applied our speed evaluation on the Logistic regression and Least Squares problems, both with L1 and L2 regularization. We measured the total time to make one epoch over data-sets of various dimensions and of various lengths, and plotted a regression line, so that we can indeed convince ourselves than our implementation is slower than a regular incremental gradient method by a small-enough constant factor. The results are plotted in Figure \ref{fig:cvx_linreg_empirical_speed}. As a summary, without mini-batching, the proximal-point method is roughly 5-10 times slower than its proximal-gradient counterpart. However, mini-batches may make an incremental gradient method computational faster by a factor of 20-150, depending on the problem. As we pointed out, this slowdown is mainly caused by the significantly costlier solution of the dual problem. Even though it appears to be a one-dimensional optimization problem, the computational complexity of evaluating the derivative of the one-dimensional objective does depend on the problem dimension. 

\begin{figure}
    \centering
    \includegraphics[width=\textwidth]{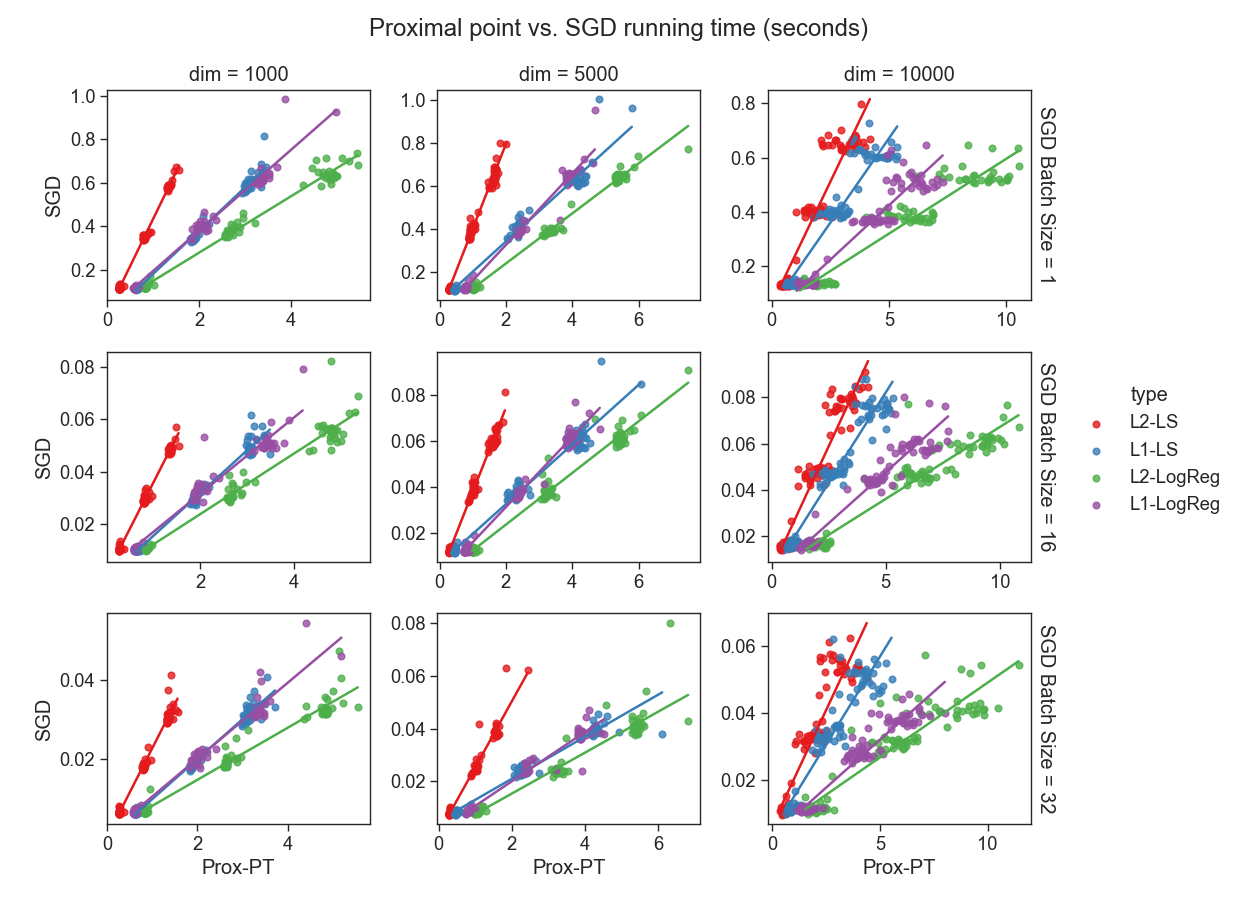}
    \caption{Execution speed evaluation of incremental proximal point. Each point is a timing of a pair of experiments on the same problem, where the $x$ coordinate is the execution time of one SGD epoch, whereas the $y$ coordinate is the execution time of one proximal point epoch. The columns are various problem dimensions, and the rows are various mini-batch sizes for the incremental proximal gradient method. The columns are various problem dimensions, and the rows are various mini-batch sizes for the incremental proximal gradient method. We can see by the first row, for example, that without mini-batching the proximal-point method is 5-10 times slower than its proximal-gradient variant. In the last row, for example, we can see that for batch sizes of 32 samples the incremental proximal point method is roughly 20-150 times slower, depending on the problem and the dimension.}
    \label{fig:cvx_linreg_empirical_speed}
\end{figure}

Regarding correctness, we again adopt a similar strategy to the convex-onto-linear setting, and solve various problem types with a variety of step-sizes, and see the achieved training loss of one epoch over a data-set of 10,000 samples in $\reals^{100}$. The results are plotted in Figure \ref{fig:cvxlinreg_stability_reproduction}

\begin{figure}
    \centering
    \includegraphics[width=\textwidth]{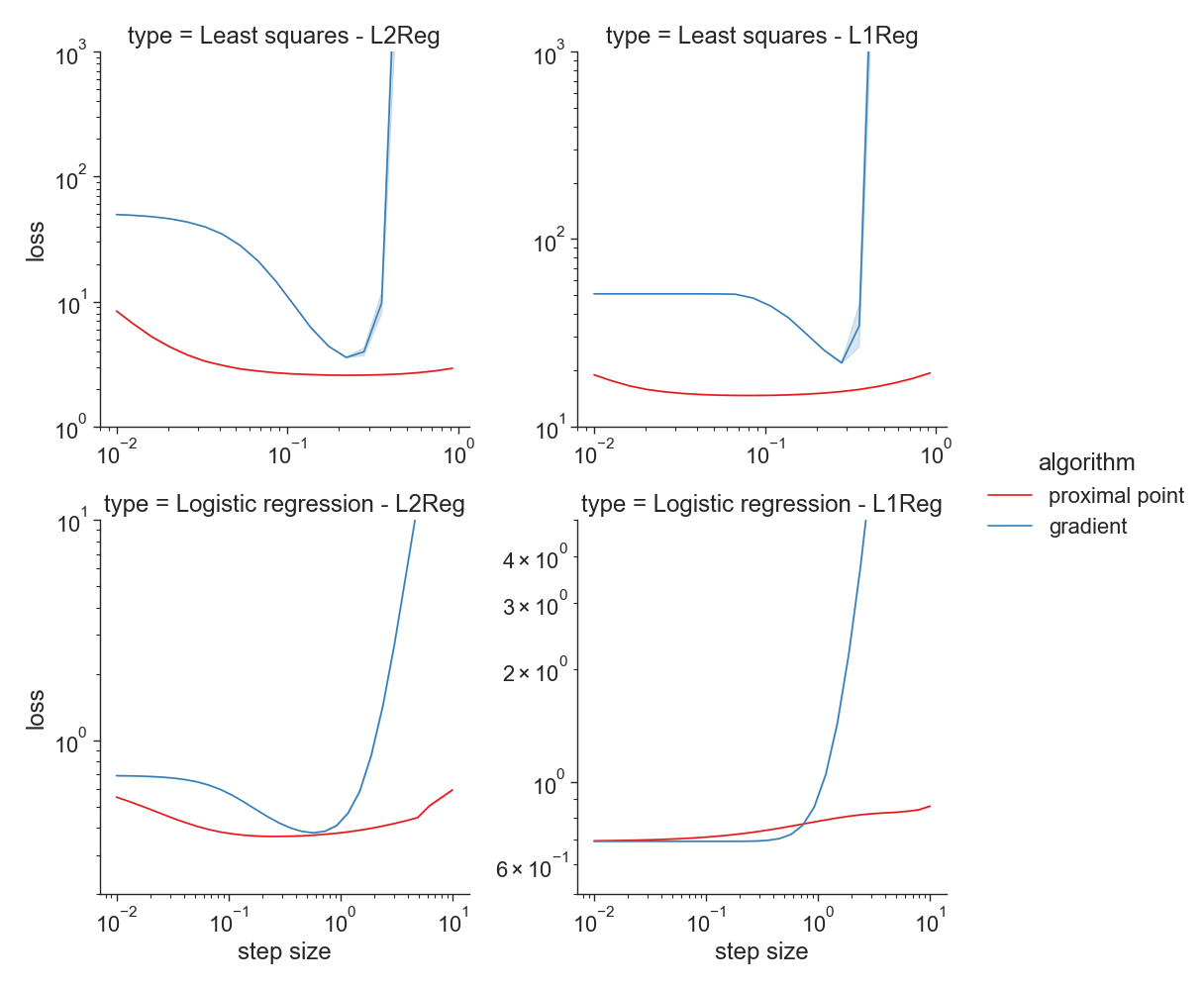}
    \caption{Reproduction of the stability results of \cite{asi_duchi_better_models}. Indeed, for both logistic regression and least squares problems, with both L1 and L2 regularization, the training loss of an incremental proximal-point method is significantly more stable w.r.t the step-size choice than the incremental proximal-gradient method, reinforcing the claim that with an incremental proximal-point method we can just "make an few educated guesses" instead of performing an extensive hyper-parameter search for the best step-size.}
    \label{fig:cvxlinreg_stability_reproduction}
\end{figure}

\subsection{Summary}
Throughout this paper, our aim is to design implementations of proximal point algorithms with two objectives in mind: (a) make the algorithms extensible by decoupling concerns into separate modules, and (b) use existing text-book concepts and software libraries to help build these modules. In the case of convex onto linear composition, we were able to decouple the function $h$ and the regularizer $r$ into separate modules which provide oracles to a central algorithm. The oracles are implemented using three textbook concepts: the convex conjugate of $h$, the Moreau envelope of $r$, and the proximal operator of $r$. Using textbook concepts allows easier extension of our library with additional functions $h$ and $r$, since we stand on the shoulders of the giants who already developed tables and calculus rules of convex conjugates and proximal operators for a variety of useful functions.

An alternative approach to using duality could be using a generic solver package, such as \texttt{CVXPY} \cite{diamond2016cvxpy}, to directly solve the proximal sub-problem. The numerical experiments of the next section show that this approach is significantly slower in practice for the setting at hand, where we aim to learn from one sample at a time, rather than from a mini-batch of samples. Indeed, our solver is roughly 5-10 times slower than SGD, whereas the CVXPY approach for small batch sizes appears to be at least 30 times slower (see Figure 5b).

Another alternatives for solving the proximal sub-problem could be primal-dual hybrid gradient methods, such as the celebrated Chambolle-Pock method \cite{chambollepock}, and their accelerated variants \cite{Chambolle2016}. However, this would result in a radically different decomposition into software components than our duality-based approach. Therefore, these approaches are out of the scope of this paper, but are worth exploring in a future work.

\section{Mini-batch convex onto linear composition}\label{sec:cvx_lin_mb}
The mini-batch convex onto linear compositions are useful when extending the scenarios described in Section \ref{sec:cvx_lin_app} to handle mini-batches of data, instead of individual data points. For example, suppose we're aiming to minimize
\[
F(\vx) = \frac{1}{N} \sum_{i=1}^N f_i(\vx),
\]
and that the size of the data-set $N$ is huge. A standard practice with incremental optimization methods, such as SGD, is to use mini-batches of items: a mini-batch $B \subseteq \{1, \dots, N \}$ is chosen in each iteration, and the model's parameters  are updated using the gradient of the function 
\[
F_B(\vx) = \frac{1}{|B|} \sum_{j \in B} f_i(\vx).
\]
The idea above is equivalent to using a linear approximation of $F_B$ in each iteration. For example, the steps of mini-batch SGD are:
\[
\vx_{t+1} = \vx_t - \eta \nabla F_B(\vx_t) = \argmin_{\vx} \left\{ F_B(\vx_t) + \langle \nabla F_B(\vx_t), \vx - \vx_t \rangle + \frac{1}{2\eta} \| \vx - \vx_t \|_2^2 \right\}.
\]
If our functions $f_i$ are of the form $h(\va^T_i \vx + b_i)$ with $h$ being a convex extended real-valued function, or if we chose to approximate $f_i$ using such functions, then by replacing the linear approximation $\vx_{t+1}$ is computed by solving
\[
\vx_{t+1} = \argmin_x \left\{ \frac{1}{|B|} \sum_{i \in B} h(\va_i^T \vx + b_i) + \frac{1}{2\eta} \|\vx - \vx_t\|_2^2 \right\}.
\]
Assuming w.l.o.g that $B = \{1, \dots, m\}$, the above is exactly of the form in Equation \eqref{eq:cvx_lin_minibatch}. As was the case with our previous derivations, convex duality plays a central role in computing the above proximal operator. We note that an alternative formulation for mini-batching could be using the proximal average \cite{prox_avg1,prox_avg2}, rather than the arithmetic average of functions. In such a formulation, the proximal operator of the mini-batch reduces to the average of the proximal operators of the individual functions, and can be computed using the tools developed in Section \ref{sec:cvx_lin}. However, we also note that the interplay of such a formulation with the stochastic optimization setting is not clear - what is the bias, if any, of the proximal average as an estimator for the mean loss? 

\subsection{A dual problem}
Embedding the vectors $\va_1^T, \dots, \va_m^T$ into the rows of the matrix $\vA$, and the scalars $b_1, \dots, b_m$ into the vector $\vb$, and introducing the auxiliary variable $\vz = \vA \vx + \vb$, our optimization problem can be equivalently written as
\[
\min_{\vx, \vz} \quad \frac{1}{m} \sum_{i=1}^m h(z_i) + \frac{1}{2\eta} \| \vx - \vx_t \|_2^2 \quad \text{s.t.} \quad \vz = \vA \vx + \vb.
\]
The corresponding dual problem aims to maximize
\begin{align*}
q(\vs) 
 &= \inf_{\vx, \vz} \left\{ \frac{1}{m} \sum_{i=1}^m h(z_i) + \frac{1}{2\eta} \| \vx - \vx_t \|_2^2 + \vs^T (\vA \vx + \vb - \vz)  \right\} \\
 &= \underbrace{ \inf_\vx \left\{ (\vA^T \vs)^T \vx + \frac{1}{2\eta} \| \vx - \vx_t \|_2^2 \right\} }_{(*)}
  + \sum_{z=1}^m \underbrace{ \inf_{z_i} \left\{ \frac{1}{m} h(z_i) - z_i s_i \right \} }_{(**)} + \vs^T \vb
\end{align*}
The term denoted by (*) above, despite its "hairy" appearance, is the minimum of a simple convex quadratic function, which is computed by comparing the gradient of the term inside the $\inf$ with zero, which leads to
\begin{equation}\label{eq:cvx_lin_minibatch_recover}
\vx = \vx_t - \eta \vA^T \vs,
\end{equation}
and the corresponding minimum is
\[
(*) = -\frac{\eta}{2} \|  \vA^T \vs\|_2^2 + (\vA \vx_t)^T \vs.
\]
The terms denoted by (**) can be equivalently written as
\[
\inf_{z_i} \left\{ \frac{1}{m} h(z_i) - z_i s_i \right \} = -\frac{1}{m} \sup_{z_i} \left\{ (m s_i) z_i - h(z_i) \right\} = -\frac{1}{m} h^*(m s_i),
\]
where, again, $h^*$ denotes the convex conjugate of $h$. To summarize, the dual aims to solve:
\[
\max_{\vs} \quad q(\vs) = -\frac{\eta}{2} \|  \vA^T \vs\|_2^2 + (\vA \vx_t + \vb)^T \vs - \frac{1}{m} \sum_{i=1}^m h^*(m s_i).
\]
\subsection{Computing the proximal operator}
Recall, that the strong duality theorem (Theorem \ref{thm:strong_duality}) ensures that $q$ has a maximizer, and that Equation \eqref{eq:cvx_lin_minibatch_recover} recovers the primal minimizer, which is what we aim to compute. Thus, computing our proximal operator amounts to the following three steps:
\begin{enumerate}
    \item Form $q(\vs) = -\frac{\eta}{2} \|  \vA^T \vs\|_2^2 + (\vA \vx_t + \vb)^T \vs - \frac{1}{m} \sum_{i=1}^m h^*(m s_i)$
    \item Solve the dual problem: find a maximizer $\vs^*$ of $q(\vs)$
    \item Recover the desired solution: $\vx_t - \eta \vA^T \vs^*$
\end{enumerate}
Note, that in the mini-batch setting, the dual problem is no longer \emph{one dimensional}, but when the size of the mini-batch $m$ is small, it is still a \emph{low dimensional} problem, whose solution should be pretty quick. As was the case in Section \ref{sec:cvx_lin}, our representation of the function $h$ amounts to a class which provides two oracles: (a) evaluate the function $h$, and (b) maximize functions of the form $\frac{1}{2} \| \vv{P} \vs\|_2^2 + \vv{c}^T \vs - \frac{1}{m} \sum_{i=1}^m h^*(m s_i)$. Let's first implement a generic optimizer, and then dive into the implementation of concrete functions $h$.
\begin{pycode}
import torch

class MiniBatchConvLinOptimizer:
    def __init__(self, x, phi):
        self._x = x
        self._phi = phi

    def step(self, step_size, A_batch, b_batch):
        # helper variables
        x = self._x
        phi = self._phi

        # compute dual problem coefficients
        P = math.sqrt(step_size) * A_batch.t()
        c = torch.addmv(b_batch, A_batch, x)

        # solve dual problem
        s_star = phi.solve_dual(P, c)
        
        # perform step
        step_dir = torch.mm(A_batch.t(), s_star)
        x.sub_(step_size * step_dir.reshape(x.shape))

        # return the mini-batch losses w.r.t the params before making the step
        return phi.eval(c)
\end{pycode}
We will shortly implement $h(z) = \half z^2$ in the \texttt{HalfSquared} class, $h(z) = \ln(1+\exp(z))$ in the \texttt{Logistic} class, and $h(z)=\max(0, z)$ in the \texttt{Hinge} class. With the above we can now, for example, solve a linear least-squares problem using mini-batches of training samples:
\begin{pycode}
x = torch.zeros(dim)
optimizer = MiniBatchConvLinOptimizer(x, HalfSquared())
dataset = get_my_dataset()
for t, (A_batch, b_batch) in enumerate(DataLoader(dataset, batch_size=32)):
    step_size = get_step_size(t)
    optimizer.step(step_size, A_batch, b_batch)

print('The model parameter vector is', x)
\end{pycode}
The concrete implementation of various functions $h$ can be found in Appendix \ref{app:cvx_lin_mb_code}. We note that the dual problem for $h(z) = \half z^2$ can be computed analytically, but for the hinge and logistic functions it cannot, and therefore we employ \texttt{CVXPY} \cite{diamond2016cvxpy, agrawal2018rewriting}, which in fact is an interface to a variety of lower-level convex optimization solvers. Since the dual problem's dimension is proportional to the mini-batch size, which is typically quite small, solving it should still be moderately fast. This is, again, the case when a combined C/Python implementation could be more efficient than pure Python code, since we could use C to implement a dedicated efficient convex solver for the dual problem of each loss.

\subsection{Empirical evaluation}
In contrast to the previous problem families, this time our proximal-point solver can handle mini-batches of data. Thus, when comparing execution speed, we compare the same mini-batch size for both our method and the incremental gradient method. Figure \ref{fig:minibatch_cvxlin_empirical_speed} contains the plots for various mini-batch sizes and two problem types - linear least-squares and logistic regression. We see that for least-squares problems, the mini-batched proximal point method is on par with the incremental gradient method. However, logistic regression problems employ the \texttt{Logistic} class which incurs the overhead of the \texttt{CVXPY} framework, and is approximately 15 times slower for mini-batches of 32 samples, and approximately 100 times slower for a mini-batch of 8 samples. The overhead of \texttt{CVXPY} is apparent, however, it isn't our aim to design and develop dedicated solvers for high-dimensional convex optimization problems. However, we encourage our readers to do so, or find a faster third-party solver, if they desire a production-grade optimizer for their model.

\begin{figure}
    \centering
    \begin{subfigure}{.9\textwidth}
        \centering
        \includegraphics[width=.8\textwidth]{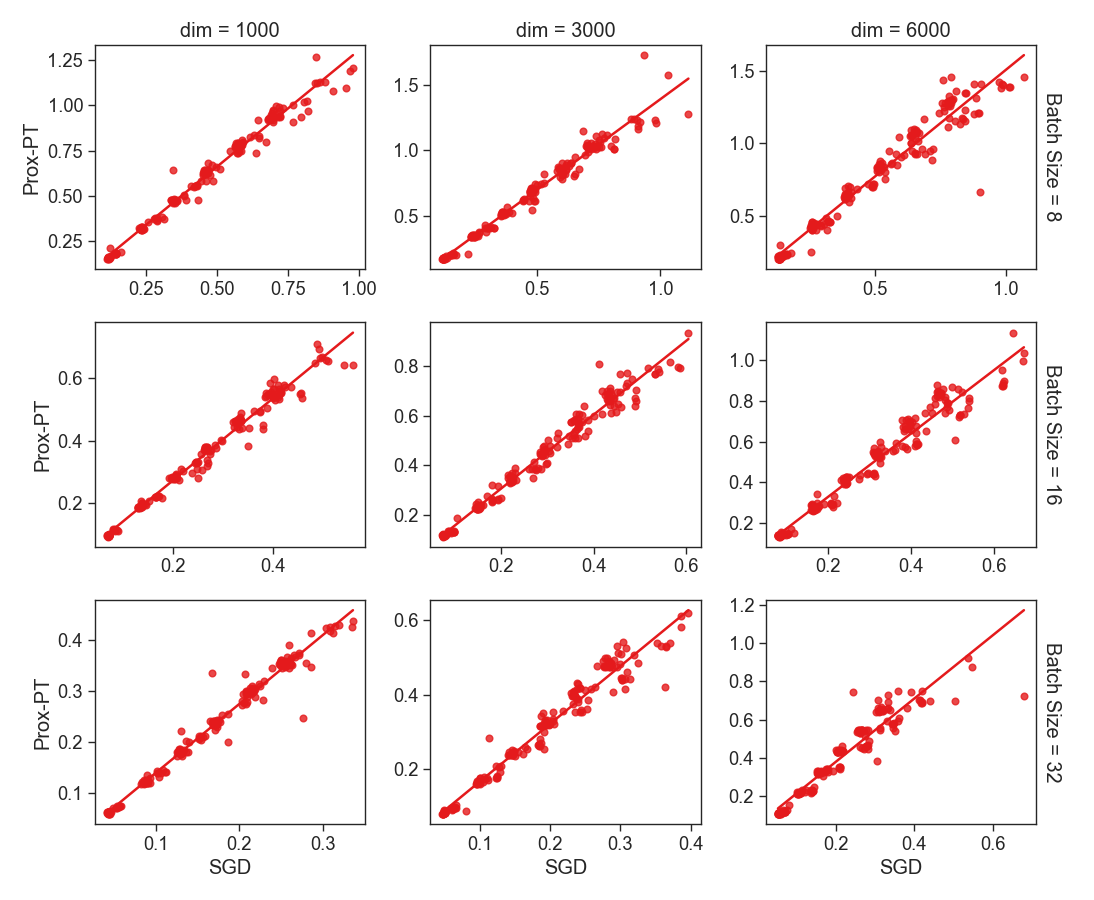}
        \caption{Least squares proximal point vs. SGD running time (seconds)}
        \label{fig:minibatch_cvxlin_empirical_speed_ls}
    \end{subfigure}
    
    \hfill
    
    \begin{subfigure}{.9\textwidth}
        \centering
        \includegraphics[width=.8\textwidth]{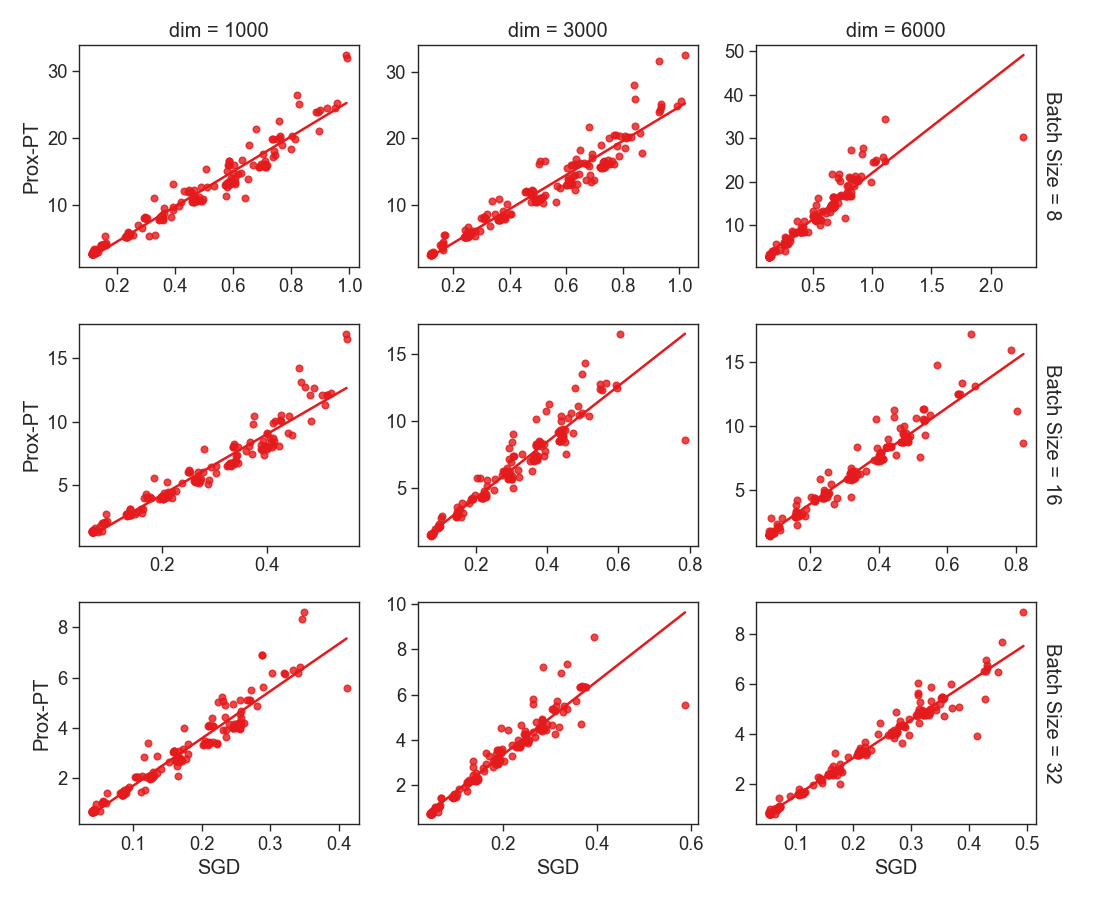}
        \caption{Logistic regression proximal point vs. SGD running time (seconds)}
        \label{fig:minibatch_cvxlin_empirical_speed_logreg}
    \end{subfigure}
    
    \caption{Execution speed evaluation of mini-batch incremental proximal point. Each point is a timing of a pair of experiments on the same problem, where the $x$ coordinate is the execution time of one SGD epoch, whereas the $y$ coordinate is the execution time of one mini-batch proximal point epoch. Points differ in running times due to generated data-set size. The columns are various problem dimensions, and the rows are various mini-batch sizes for the incremental gradient method. The corresponding line is a least-squares regression line, whose slope allows to appreciate the ratio between the SGD and proximal point running times. For least-squares problems, the mini-batched proximal point method is on par with the incremental gradient method. However, logistic regression problems employ the \texttt{Logistic} class which incurs the overhead of the \texttt{CVXPY} framework and of a generic conic solver, and is approximately 15 times slower for mini-batches of 32 samples, and approximately 100 times slower for a mini-batch of 8 samples, where CVXPY's overhead is more significant.}
    \label{fig:minibatch_cvxlin_empirical_speed}
\end{figure}

And again, to see that we indeed harvest the fruits of a proximal-point algorithm, we conduct a stability experiment, where we solve logistic regression and least-squares problems, and expect to see the resulting algorithm being much more stable, with respect to the step-size choice, than an incremental gradient method. The results are plotted in Figure \ref{fig:minibatch_cvxlin_empirical_stability}, where the stability is apparent - we indeed obtain a low training loss value for a large range of step-sizes, in contrast to the incremental gradient method, where we need to "pinpoint" the correct step-size to obtain good performance.

\begin{figure}
    \centering
    \includegraphics[width=\textwidth]{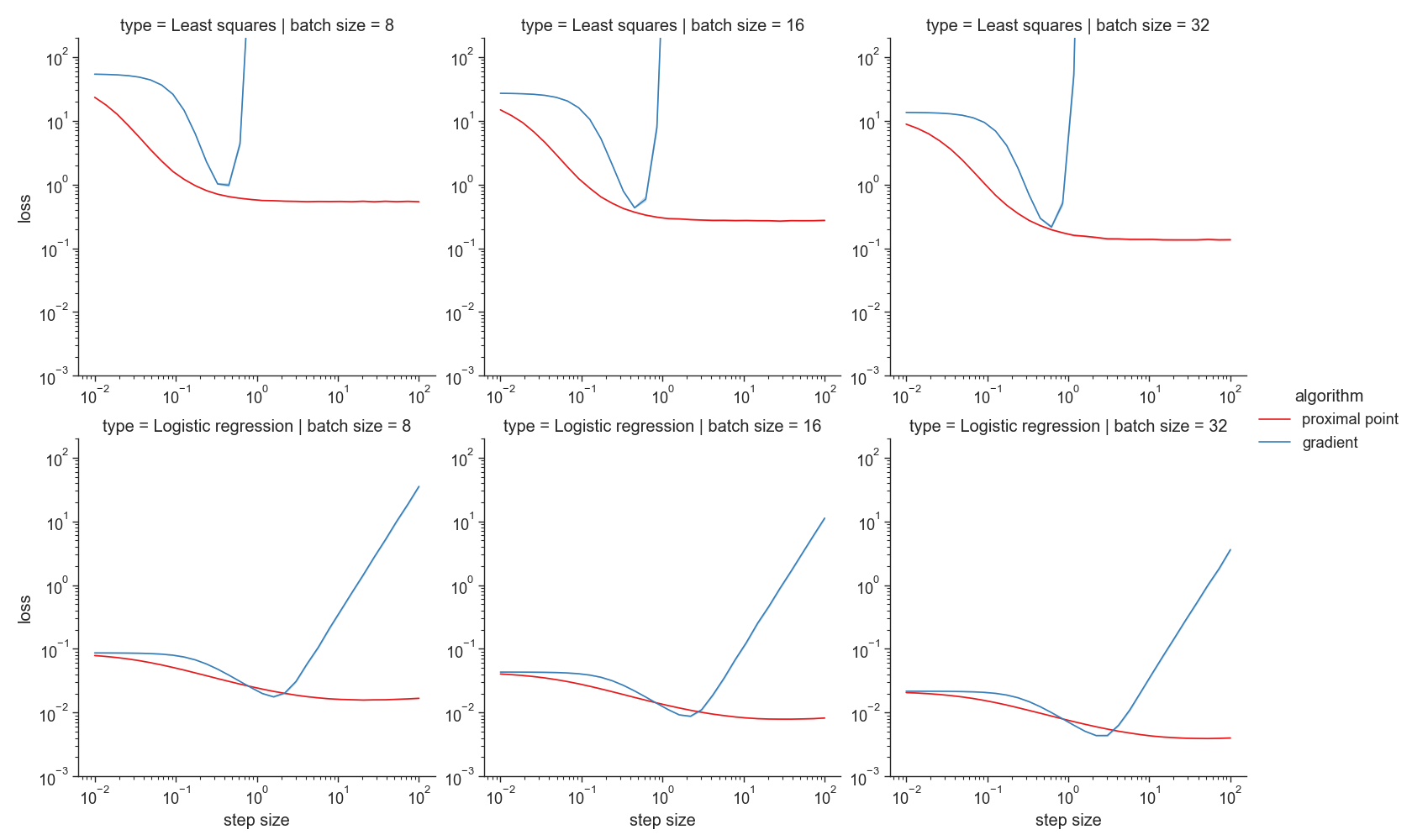}
    \caption{Results of solving logistic regression and linear least-squares problems using an incremental gradient method and our proximal-point with mini-batches implementation.}
    \label{fig:minibatch_cvxlin_empirical_stability}
\end{figure}

\subsection{Summary}
Duality played a central role here as well, but instead of reducing the proximal operator problem to a one-dimensional problem, we reduced the proximal operator to a, hopefully, low dimensional problem whose dimension is the size of the mini-batch. In practice, mini-batches typically have less than $128$ samples, and thus our dual problems are of very low dimensions and can be solved extremely quickly. There is, of course, an additional overhead incurred by using a commodity optimization package such as CVXPY, and a dedicated solver for each dual problem could be substantially faster. However, writing robust and efficient convex optimziation solvers is out of the scope of this paper.

\vskip 0.2in
\bibliographystyle{spmpsci}
\bibliography{proxpt}

\section{Statements and Declarations}
The authors declare that no funds, grants, or other support were received during the preparation of this manuscript. \\
The authors have no relevant financial or non-financial interests to disclose.

\clearpage
\appendix

\section{Code for regularized convex onto linear compositions}\label{app:cvx_lin_reg_code}

\subsection{The \texttt{HalfSquared} class}
For $h(z) = \frac{1}{2} z^2$, looking at the Table \ref{tab:conjugates}, we see that $h^*(s) = \frac{1}{2} s^2$ with $\dom(h^*) = \reals$, i.e. its conjugate is itself. Since the domain is non-compact and open, and $h^*$ is strictly convex, we perfectly fit the non-compact case. We will use the sequence $-1, -2, -2^2, -2^3, \dots$ for the lower bounds, and $1, 2, 2^2, 2^3, \dots$ for upper bounds. The implementation is below:
\begin{pycode}
from itertools import count

class HalfSquared:
    def eval(self, z):
        return (z ** 2) / 2
    
    def conjugate_has_compact_domain(self):
        return False
        
    def lower_bound_sequence(self):
        return (-(2 ** j) for j in count())
    
    def upper_bound_sequence(self):
        return ((2 ** j) for j in count())
        
    def conjugate(self, s):
        return (s ** 2) / 2
        
    def conjugate_prime(self, s):
        return s
\end{pycode}

\subsection{The \texttt{Logistic} class}
For $h(z) = \ln(1 + \exp(z))$, looking at Table \ref{tab:conjugates}, we see that $h^*(s) = s \ln(s) + (1 - s) \ln(1 - s)$ with the convention that $0 \ln(0) = 0$, and that the domain of $h^*$ is the compact interval $[0, 1]$.
Based on the above, we implement the \texttt{Logistic} class below:
\begin{pycode}
import math

class Logistic:
    def eval(self, z):
        return math.log1p(math.exp(z))

    def conjugate_has_compact_domain(self):
        return True
        
    def domain(self):
        return (0, 1)
        
    def conjugate(self, s):
        def entr(u):
            if u == 0:
                return 0
            else:
                return u * math.log(u)
        
        return entr(s) + entr(1 - s)
\end{pycode}
Note that since we're not going to be searching for an initial interval, we don't need to implement the methods returning upper bound and lower bound sequences.

\subsection{The \texttt{Hinge} class}
For $h(z) = \max(0, z)$, looking at Table \ref{tab:conjugates}, we see that $h^*(s)$ is just the indicator of the interval $[0, 1]$, and we again fall into the compact domain case. Below is the implementation:
\begin{pycode}
import math

class Hinge:
    def eval(self, z):
        return max(0, z)

    def conjugate_has_compact_domain(self):
        return True
        
    def domain(self):
        return (0, 1)
        
    def conjugate(self, s):
        if s < 0 or s > 1:
            return math.inf
        else:
            return 0
\end{pycode}

\subsection{The \texttt{L2Reg} class}
First, note that for all regularizers we need to be able to compute their Moreau envelope and their proximal operator. Hence, we first define a common base class:
\begin{pycode}
from abc import ABC, abstractmethod

class Regularizer(ABC):
    @abstractmethod
    def prox(self, eta, x):
        pass

    @abstractmethod
    def eval(self, x):
        pass

    def envelope(self, eta, x):
        prox = self.prox(eta, x)
        result = self.eval(prox) + 0.5 * (prox - x).square().sum() / eta
        return result.item()
\end{pycode}
Now we can use it to implement our \texttt{L2Reg} class, which represents $r(\vx) = \frac{\mu}{2} \|\vx\|_2^2$, using the proximal operator in Table \ref{tab:prox_operators}.
\begin{pycode}
class L2Reg(Regularizer):
    def __init__(self, mu):
        self._mu = mu
        
    def prox(self, eta, x):
        return x / (1 + self._mu * eta)
        
    def eval(self, x):
        return self._mu * x.square().sum() / 2.
\end{pycode}

\subsection{The \texttt{L1Reg} class}
Using the \texttt{Regularizer} base class above, and on Table \ref{tab:prox_operators}, we can also implement the \texttt{L1Reg} class for representing $r(\vx) = \mu \|\vx\|_1$.
\begin{pycode}
from torch.nn.functional import softshrink

class L1Reg(Regularizer):
    def __init__(self, mu):
        self._mu = mu
        
    def prox(self, eta, x):
        softshrink(x, eta * self._mu)
        
    def eval(self, x):
        return self._mu * x.abs().sum()
\end{pycode}

\subsection{The \texttt{L2NormReg} class}
The following class represents the $r(\vx) = \mu \| \vx\|_2$ class.
\begin{pycode}
from torch.linalg import norm

class L2NormReg(Regularizer):
    def __init__(self, mu):
        self._mu = mu
        
    def prox(self, eta, x):
        nrm = norm(x)
        eta = eta * self._mu
        return (1 - eta / max(eta, nrm)) * x
        
    def eval(self, x):
        return self._mu * norm(x)
\end{pycode}

\section{Code for mini-batch of convex onto linear compositions}\label{app:cvx_lin_mb_code}

\subsection{The \texttt{HalfSquared} class}
For $h(z) = \frac{1}{2} z^2$, we have $h^*(s) = \frac{1}{2} s^2$. Hence, our dual problem is of the form
\begin{align*}
q(\vs) 
 &= - \half \| \vv{P} \vs\|_2^2 + \vv{c}^T \vs - \frac{m}{2} \| \vs \|_2^2 \\
 &= -\half \vs^T \left( \vv{P}^T \vv{P} + m \vv{I} \right) \vs + \vv{c}^T \vs,
\end{align*}
where $\vv{I}$ is the identity matrix of the appropriate size. It's a simple strictly concave quadratic function, which can be minimized by equating its gradient with zero:
\[
\nabla q(\vs) = -(\vv{P}^T \vv{P} + m \vv{I}) \vs  + \vv{c} = 0.
\]
Re-arranging the above equation leads to the maximizer
\[
\vs = (\vv{P}^T \vv{P} + m \vv{I})^{-1} \vv{c}.
\]
It's also easy to see that $\vv{P}^T \vv{P} + m \vv{I}$ is a symmetric positive-definite matrix, and thus $\vs$ can be obtained using the well-known Cholesky decomposition. Fortunately, \texttt{PyTorch} has all the necessary machinery to do exactly that.
\begin{pycode}
class HalfSquared:
    def solve_dual(self, P, c):
        m = P.shape[1]  # number of columns = batch size

        # construct lhs matrix P* P + m I
        lhs_mat = torch.mm(P.t(), P)
        lhs_mat.diagonal().add_(m)

        # solve positive-definite linear system using Cholesky factorization
        lhs_factor = torch.cholesky(lhs_mat)
        rhs_col = c.unsqueeze(1)  # make rhs a column vector, so that cholesky_solve works
        return torch.cholesky_solve(rhs_col, lhs_factor)

    def eval(self, lin):
        return 0.5 * (lin ** 2)
\end{pycode}

\subsection{The \texttt{Logistic} class}
In direct contrast to the case of the half-squared function, for $h(z) = \ln(1+\exp(z))$ with $h^*(s) = s \ln(s) + (1 - s) \ln(1 - s)$ we don't have a formula for computing a maximizer $s^*$. However, convex optimization is a mature technology, and a variety of extremely fast and efficient software packages exists to do exactly that - minimize convex functions, or equivalently, maximize concave functions. In this paper we'll use one such package, \texttt{CVXPY} \cite{diamond2016cvxpy, agrawal2018rewriting}, which in fact is an interface to a variety of lower-level convex optimization solvers. 
\begin{pycode}
import torch
import cvxpy as cp

class Logistic:
    def solve_dual(self, P, c):
        # extract information and convert tensors to numpy. CVXPY
        # works with numpy arrays
        dtype = P.dtype
        m = P.shape[1]
        P = P.data.numpy()
        c = c.data.numpy()

        # define the dual optimization problem using CVXPY
        s = cp.Variable(m)
        objective = 0.5 * cp.sum_squares(P @ s) - \
            cp.sum(cp.multiply(c, s)) - \
            (cp.sum(cp.entr(m * s)) + cp.sum(cp.entr(1 - m * s))) / m

        prob = cp.Problem(cp.Minimize(objective))

        # solve the problem, and extract the optimal solution
        prob.solve()
        
        # recover optimal solution, and ensure it's cast to the same type as 
        # the input data.
        return torch.tensor(s.value).to(dtype=dtype).unsqueeze(1)

    def eval(self, lin):
        return torch.log1p(torch.exp(lin))
\end{pycode}

\subsection{The \texttt{Hinge} class}
As was the case with the \texttt{Logistic} class, there is no closed-form solution for solving the dual. The conjugate of $h(z) = \max(0, z)$ is the indicator of the interval $[0, 1]$, and thus the dual problem aims to solve
\[
\max_\vs \quad \half \| \vv{P} \vs\|_2^2 + \vv{c}^T \vs  \quad \text{s.t.} \quad 0 \leq s_i \leq \frac{1}{m}
\]
The corresponding Python code using CVXPY is below.
\begin{pycode}
import torch
import torch.nn.functional
import cvxpy as cp

class Hinge:
    def solve_dual(self, P, c):
        # extract information and convert tensors to numpy. CVXPY
        # works with numpy arrays
        dtype = P.dtype
        m = P.shape[1]
        P = P.data.numpy()
        c = c.data.numpy()

        # define the dual optimization problem using CVXPY
        s = cp.Variable(m)
        objective = 0.5 * cp.sum_squares(P @ s) - cp.sum(cp.multiply(c , s))

        constraints = [s >= 0, s <= 1. / m]
        prob = cp.Problem(cp.Minimize(objective), constraints)

        # solve the problem, and extract the optimal solution
        prob.solve()
        return torch.tensor(s.value).to(dtype=dtype).unsqueeze(1)

    def eval(self, lin):
        return torch.nn.functional.relu(lin)

\end{pycode}

\end{document}